\documentclass{article}

\usepackage{microtype}
\usepackage{graphicx}
\usepackage{subfigure}
\usepackage{booktabs} 

\usepackage{hyperref}

\usepackage[accepted]{packages/icml2021}

\icmltitlerunning{Approximation Theory of CNN for Times Series}

\usepackage{epsfig,epsf,psfrag}
\usepackage{amsfonts}
\usepackage{mathtools}
\usepackage[inline]{enumitem}
\usepackage{bm}
\usepackage{amsthm}

\usepackage{comment}
\usepackage{hyperref}
\DeclarePairedDelimiter\abs{\lvert}{\rvert}
\DeclarePairedDelimiter\norm{\lVert}{\rVert}
\newcommand{\Conv}{\mathop{\scalebox{1.7}{\raisebox{-0.2ex}{$\ast$}}}}%
\newcommand{\set}[1]{\left\lbrace #1 \right\rbrace}
\usepackage{ulem}
\usepackage{afterpage}
\usepackage{lipsum}

\newcommand{\hcnn}{{{\mathcal H}}_{\text{CNN}}}

\newcommand{\hrnn}{\mathcal H_{\text{RNN}}}
\newcommand{\cs}{\mathcal C}
\newcommand\figref{Figure~\ref}
\newcommand{\tensor}[1]{\bm{\mathcal {#1}}}

\newcommand{\rad}[1]{r(#1)}

\newcommand{\depen}[1]{^{(#1)}}

\newcommand{\bhh}{\bm {\hat H}}
\newcommand{\rh}{\bm \rho\depen{\bm H}}
\newcommand{\rhn}{\rho\depen{\bm H}}

\newcommand{\rhh}{\bm \rho\depen{\bhh}}

\newcommand{\ten}[1]{T_{l^K}(#1)}

\DeclareMathOperator{\sgn}{sgn}

\DeclareMathOperator{\rank}{rank}

\newtheorem{theorem}{Theorem}
\newtheorem{lemma}[theorem]{Lemma}
\newtheorem{definition}{Definition} 
\newtheorem{example}{Example}
\newtheorem{remark}{Remark}

\newtheorem{theoremAppendix}{Theorem}[section]
\newtheorem{propositionAppendix}[theoremAppendix]{Proposition}
\newtheorem{corollaryAppendix}[theoremAppendix]{Corollary}
\newtheorem{remarkAppendix}{Remark}[section]
\newtheorem{definitionAppendix}[theoremAppendix]{Definition} 
\newtheorem{exampleAppendix}{Example}[section]

\usepackage{xparse}
\usepackage{amsmath}
\usepackage{environ}

\begin{document}

\twocolumn[
\icmltitle{Approximation Theory of Convolutional Architectures for Time Series Modelling}



\icmlsetsymbol{equal}{*}

\begin{icmlauthorlist}
\icmlauthor{Haotian Jiang}{nus}
\icmlauthor{Zhong Li}{pku}
\icmlauthor{Qianxiao Li}{nus,ihpc}
\end{icmlauthorlist}

\icmlaffiliation{nus}{Department of Mathematics, National University of Singapore}
\icmlaffiliation{pku}{School of Mathematical Science, Peking University}
\icmlaffiliation{ihpc}{Institute of High Performance Computing, A*STAR, Singapore}

\icmlcorrespondingauthor{Qianxiao Li}{qianxiao@nus.edu.sg}

\icmlkeywords{Machine Learning, ICML}

\vskip 0.3in
]



\printAffiliationsAndNotice{} 

\begin{abstract}
We study the approximation properties of convolutional architectures applied to time series modelling, which can be formulated mathematically as a functional approximation problem. In the recurrent setting, recent results reveal an intricate connection between approximation efficiency and memory structures in the data generation process. In this paper, we derive parallel results for convolutional architectures, with WaveNet being a prime example. Our results reveal that in this new setting, approximation efficiency is not only characterised by memory, but also additional fine structures in the target relationship. This leads to a novel definition of spectrum-based regularity that measures the complexity of temporal relationships under the convolutional approximation scheme. These analyses provide a foundation to understand the differences between architectural choices for time series modelling and can give theoretically grounded guidance for practical applications.

\end{abstract}

\section{Introduction}


While deep learning has evolved to be a powerful tool to model temporal relationships, the choice of architectures significantly affects performance. Classical recurrent neural networks (RNNs) and recent adaptation of convolutional neural networks (CNNs) to time series data represent two popular classes of architectural choices. Empirical works show that depending on the application setting, either one can have an advantage over the other \citep{DBLP:journals/corr/0001KYS17,Bai2018,BANERJEE201979}. However, a concrete theoretical understanding of how such differences arise remains largely unexplored.

The recent work of \citet{li2021on} formulated the temporal
modelling task into a functional analysis problem,
and uncovered relationships between approximation efficiency of RNNs
and memory structures of the target relationship.
In this paper, we develop parallel results for CNNs in the functional
approximation setting.
On the one hand, these results provide approximation guarantees and rate estimations
for CNNs applied to time series modelling.
On the other hand, they allow us to concretely understand the differences
between CNNs and RNNs in terms of their approximation capabilities for
temporal relationships, serving as a starting point to bridge theories and applications.

Our main contributions are:
\begin{enumerate}
	\item We develop universal approximation results for CNNs when applied to model temporal relationships in the linear setting,
	which shows that the approximation efficiency is characterised by both memory structures and a certain spectrum-based regularity
	of the target relationship.
	\item We make rigorous comparisons between CNNs and
	RNNs in the approximation setting, where we show that the targets that can be easily approximated are completely different for these two architectures. This provides a theoretical foundation for building principled model selection strategies for deep learning in dynamical settings.
\end{enumerate}

The paper is organised as follows.
We discuss related work in Section \ref{sec:related_work}.
In Section \ref{sec:formulation}, we formulate the approximation problem precisely and introduce the RNN and CNN hypothesis spaces under this formulation.
The approximation results of CNNs are presented in Section \ref{sec:approximation}.
In Section \ref{sec:comparison}, we make a concrete comparison between CNNs and RNNs in terms of approximation capabilities for different classes of target temporal relationships.
The proofs of all presented results are found in the appendix.

\subsection{Notations and definitions.}
We use boldface letters $\bm x$ 
to denote discrete temporal sequences, and
$x(t)$ to denote the vector/scalar value of the sequence at the time index $t$.
We use $\bm x_{[t_1,t_2]}$ to denote values of the sequence $\bm x$
from $t_1$ to $t_2$, including the end points.
Subscripts such as $\bm x_i$, $\bm x_m$ are used to specify dimensional indices.
All other dependencies are written as superscripts, e.g.
$\bm x\depen{K}$ means that the temporal sequence $\bm x$ depends on $K$. For a discrete sequence $\bm \rho: \mathbb N \to \mathbb R^d$, define the radius of $\bm \rho$ by 
$\rad{\bm \rho} = \sup\set{ s:\rho(s) \neq 0}$.
We refer $[0, r(\bm \rho)]$ as its support.
We use $\abs{\cdot}$ to denote the Euclidean norm of a vector.

We set the notation for the central operation under study, namely the dilated convolution for discrete sequences.
\begin{definition}
	Let $\bm f: \mathbb Z \to \mathbb R^d$, $\bm g: \mathbb Z \to \mathbb R^d$ be two discrete sequences.  Define the following discrete dilated convolution:
	\begin{equation}
		(\bm f \Conv{}_{l} \, \bm g)(t) = \sum_{s\in \mathbb Z}f(s)^\top g(t-ls).
	\end{equation}
	When $l=1$, this is the usual convolution, and we denote it by $\Conv$.
\end{definition}

\section{Related work}
\label{sec:related_work}

In this section, we discuss some related work to the results presented here.
On the application side, convolution-based models are gradually becoming popular for a variety of time series applications.
For example, WaveNet \citep{Oord2016} is a generative model for processing and generating raw audio.
Since its emergence, WaveNet and its variants
\citep{pmlr-v80-oord18a,8462431,pmlr-v80-kalchbrenner18a,8683143,DBLP:journals/corr/PaineKCZRHH16}
have been successfully applied to a variety of time series modelling problems in natural language processing.
A more detailed review can be found in \citet{boilard2019a}.
While convolutional architectures have become an important alternative to classical recurrent neural networks,
these investigations focus on applications, lacking theoretical results to understand the origin of their superior performance on practical tasks.

On the theoretical side, a number of universal approximation results for CNNs have been obtained in the image processing setting. For example, \citet{ZHOU2020319,ZHOU2020787} prove that (simplified) deep CNNs are universal approximators, and the approximation rate is characterised by Sobolev norms of target functions.
\citet{Bao2019ApproximationAO} gives a theoretical interpretation on the reason why the state-of-the-art CNNs can achieve high classification accuracy.
It is shown that the hierarchical compositional structures in target functions can help to (exponentially) reduce the parameters needed compared with fully connected neural networks achieving the same approximation accuracy.
Beyond the approximation theory, \citet{pmlr-v97-oono19a} studies the estimation error rates of a ResNet-type of CNN applied to the Barron and H\"older classes. All of these works share one thing in common, namely the targets to be approximated are functions defined in finite space domains, i.e. images. By contrast, our approximation results here are for the targets defined in infinite time domains, where the memory of data plays an important role. We will see that this leads to very different approximation problems.

On the recurrent front, the approximation properties of RNNs have been investigated in a number of works, including both in discrete time \citep{Matthews1993ApproximatingNF,doya1993universality,10.1007/11840817_66,schafer2007recurrent} and continuous time \citep{FUNAHASHI1993801,841860,article,maass2007computational,10.1007/978-3-642-04277-5_60}. Most of these study settings where the target relationships are generated from hidden dynamical systems (in the form of difference or differential equations).
Recently, \citet{li2021on} studies the general setting where the target relationships are represented as functionals, and derived the
approximation theory of RNNs that revealed the connection between approximation and memory: it takes an exponentially large number of neurons to approximate the target with memory that decays slowly.
This sheds some light on the interaction of the structure of RNNs and the nature of relationships to be captured.
Following this line of enquiry, the purpose of the present paper is to develop parallel approximation results
for CNNs to highlight such interactions.
In particular, this work complements previous theoretical analyses in the recurrent setting and allows one to characterise the key difference between recurrent and convolutional approaches for time series modelling.



\section{Problem Formulation}\label{sec:formulation}

\subsection{Functional formulation of supervised learning for temporal data}

The temporal supervised learning task can be mathematically formulated as follows.
Suppose we are given an input sequence
$\bm x$ indexed by a time parameter $t \in \mathbb Z$.
We want to predict the corresponding output sequence
$\bm y$
at each time step, or up to some terminal time step $y(T)$.
The mapping between $\bm x$ and $\bm y$ can be described by a sequence of functionals:
\begin{equation}
	\set{y(t) = H_t(\bm x): t\in \mathbb Z}.
\end{equation}
That is, the output at each time step may potentially depend on the entire input sequence through a time-dependent functional.
The goal is to learn this sequence of functionals $\bm H = \set{H_t : t \in \mathbb Z}$.

As an illustrative example, we can consider a standard benchmark example known as the adding problem \citep{doi:10.1162/neco.1997.9.8.1735}: each output $y(t)$ is equal to the weighted cumulative sum of inputs $x(s)$ for $s \leq t$, i.e. $y(t) = \sum_{-\infty}^{t} \alpha(s)x(s)$, where $\alpha$ denotes the weights. In this case, the sequence of functionals $\{ H_t(\bm x) = \sum_{-\infty}^{t} \alpha(s) x(s) : t \in \mathbb{Z}\}$ can be viewed as the target temporal relationship, or ground truth.

In supervised learning, we need to define the input space, output space and concept space  precisely.
We note that vector-valued discrete temporal sequences can also
be understood as functions from a time index set $\mathcal I$ into the real vectors.
We denote this set of functions by $\bm c(\mathcal I, \mathbb R^d)$. 
In this paper, $\mathcal I$ is taken to be $\mathbb Z$ or $\mathbb N$ depending on the specific setting.
For $\bm x \in \bm c(\mathcal I, \mathbb R^d)$, we define its norm by
$\norm{\bm x}_{2} :=  \sqrt{ \sum_{s\in \mathcal I} \abs{x(s)}^2}$.
We use $\bm c_0(\mathbb N, \mathbb R^d)$ to denote the sequence that converges to 0 at infinity. 

Define the input space by 
\begin{equation}
	\mathcal X = \Big\{ \bm x \in  \bm c(\mathbb Z, \mathbb R^d):
	\sum_{s \in \mathbb Z} \abs{x(s)}^2 < \infty \Big\}.
\end{equation}
This is the usual $\ell^2$ sequence space, which is a Banach space with the norm $\norm{\bm x}_{\mathcal X} := \norm{\bm x}_{2}= \sqrt{ \sum_{s\in \mathbb Z} \abs{x(s)}^2}$.

Since vector-valued outputs can be handled by considering each dimension individually, we can restrict our attention to the case where the output time series are real-valued.
That is, define the output space 
\begin{equation}
	\mathcal Y = \bm c(\mathbb Z, \mathbb R).
\end{equation}

\subsection{RNN and CNN hypothesis spaces }

We now introduce the RNN and CNN architectures in the functional approximation language, which leads to different types of hypothesis spaces.
We start with the recurrent setting.
The simplest recurrent neural network with the linear readout layer is given by
\begin{equation}\label{eq: rnndynamics}
		\begin{aligned}
		    h{(t+1)} &= \sigma( Wh(t) + U x(t) ),\\
		    \hat y(t) &= c^{\top} h(t), \\
		   \text{with } c \in \mathbb R^{m} , W &\in \mathbb R^{m\times m}, U \in \mathbb R^{m\times d}.
		\end{aligned}
\end{equation}
Here, $h \in \mathbb R^m$ is the hidden state and $m$ denotes the width of RNNs, which determines the model complexity.
Note that we do not include a bias term here because it can be absorbed into the hidden state $h$.


This dynamics defines a family of functionals
\begin{equation}
\begin{aligned}
	\hrnn^{(m)} :=  &\Big\{\bhh: \hat H_t = \hat y(t) \text{ solves (\ref{eq: rnndynamics})
	} \\
	  &~~\text{with } \text{$c \in \mathbb R^{m}$}, \text{$W \in \mathbb R^{m\times m}$, $U \in \mathbb R^{m\times d}$}\Big\}.
\end{aligned}
\end{equation}

The hypothesis space for one layer RNNs with arbitrary depth is defined by
\begin{equation}
	\hrnn := \bigcup_{m \in \mathbb N_+}\hrnn^{(m)}.
\end{equation}
In essence, RNNs can be understood as a way to represent functionals by introducing hidden dynamical systems, with the predicted outputs being observations.

Recent approaches based on convolutional architectures present an entirely different class of methods to parameterise the functionals. Concretely, a convolution based temporal sequence model with $K$ layers and $M_k$ channels at layer $k$ is given by
\begin{equation}\label{eq: CNNdynamics}
		\begin{aligned}
		\bm h_{0,i} &= \bm x_i,\\
		\bm h_{k+1,i} &= \sigma \left(\sum_{j=1}^{M_k} {\bm w}_{kji} \Conv{}_{d_k} \bm h_{k,j} \right),\\
		\bm {\hat y} &= \bm h_K.
		\end{aligned}
\end{equation}
Here, $\bm x_{i}$ is the $i^{th}$ dimension of $\bm x$, and
${\bm w}_{kji}$ is the filter from channel $j$ at layer $k$ to the channel $i$ at layer $k+1$.
All the filters have a size $l\geq2$, i.e. $r({\bm w}_{kji}) = l \ge 2$. 
In the following discussion, we assume that the dilation rate satisfies $d_k = l^k$.
That is, the dilation rate increases exponentially for each layer to achieve 
an exponentially large receptive field.
This is the standard practice for dilated convolutional structures \citep{oord2016wavenet,yu2016multiscale}.

The CNN model (\ref{eq: CNNdynamics}) also defines a family of functionals
\begin{equation}
\begin{aligned}
	\hcnn\depen{l,K,\set{M_k}} := & 
	\Big\{\bhh : \hat H_t =  \hat y(t) \text{ according to  \eqref{eq: CNNdynamics}}
	\\&~~\text{with $\bm w_{kji} \in \mathbb R^l$}\Big\}.
\end{aligned}
\end{equation}
For any $l\ge 2$ and $d_k = l^k$,
the hypothesis space for CNNs with arbitrary depth and number of channels is defined as
\begin{equation}
	\hcnn\depen{l} = \bigcup_{K \in \mathbb N_+}   \bigcup_{
		\set{M_k} \in \mathbb N_+^K}
	 \hcnn\depen{l,K,\set{M_k}}.
\end{equation}
\section{Approximation theory for convolutional structures}
\label{sec:approximation}

In this section, we study the approximation properties of $\hcnn\depen{l}$.
Our main results consist of two parts.
First, we prove that $\hcnn\depen{l}$ is dense in appropriate concept spaces.
Next, we derive explicit upper and lower bounds for the approximation rate,
which depends on the depth and width of CNNs, together with appropriate notions of complexity of the target functional family.
This is an important result that characterises the kind of targets that can be well-approximated by CNNs using a small number of parameters.

To make analysis amenable, we consider the linear setting.
This is non-trivial for two reasons: first, the key feature we would like to study for time series applications is the dependence on time, for which non-linearity still plays an important role even in the case of linear activations. For example, we will see that the Riesz representation of a linear functional is nonlinear in time.
The approximation results for RNNs~\citep{li2021on} has already demonstrated this.
Second, since the approximation results for linear RNNs have been derived in~\citet{li2021on},
by adopting the same setting we can study explicitly the differences between RNNs and CNNs. Under the assumption that the activation $\sigma$ is linear, the RNN and CNN hypothesis spaces can be further simplified as
\begin{equation}\label{eq: linearhRNN}
\begin{aligned}
	\hrnn^{(m)} := &\Big\{\bhh: \hat H_t(\bm x) = \sum_{s\in\mathbb N} c^\top W^{s-1}Ux(t-s),\\
	 &~~\text{$c \in \mathbb R^{m} , W \in \mathbb R^{m\times m}$, $U \in \mathbb R^{m\times d}$}\Big\},
\end{aligned}
\end{equation}
\begin{equation}\label{eq: linearhcnn}
\begin{aligned}
	\hcnn^{(l,K,\set{M_k})} := &\Big\{\bhh: 
	\hat H_t(\bm x) = \sum_{s\in \mathbb N} \rho\depen{\bhh}(s)^\top x(t-s), \\
	&\rhh (s) = \sum_{i_{K-1}=1}^{M_{K-1}}\bm w_{K-1,i_{K-1},t}\Conv{}_{l^{K-1}}\\
	&\sum_{i_{K-2}=1}^{M_{K-2}}\bm w_{K-2,i_{K-2},i_{K-1}}\Conv{}_{l^{K-2}}\\
	&\cdots \Conv{}_{l^{2}}\sum_{i_1=1}^{M_{1}}\bm w_{1 i_1 i_2} \Conv{}_{l^1} \bm w_{0 i_{1}} (s)
	\Big\}.
\end{aligned}
\end{equation}
We will also restrict to a target functional family (concept space)
that is consistent with the linear activation setting.
The following definition is introduced in \citet{li2021on}.

\begin{definition}\label{def:concept_space}
Let $\bm H = \set{H_t:t\in\mathbb Z}$ be a sequence of functionals which satisfies
\begin{enumerate}
	\item $\bm H$ is causal if it does not depend on the future inputs:
	for any $\bm x_1,  \bm x_2 \in \mathcal X$ and any $t \in \mathbb Z$ such that
	$
		x_1(s) = x_2(s) \ \text{ for all } s \leq t
	$,
	the output satisfies $H_t(\bm x_1) = H_t(\bm x_2)$;

	\item $H_t \in \bm H$ is a continuous linear functional if for any $\bm x_1,  \bm x_2 \in \mathcal X$
	and $\lambda_1, \lambda_2 \in \mathbb R$,
	\begin{equation}
		\begin{gathered}
		H_t(\lambda_1\bm x_1 + \lambda_2\bm x_2) = \lambda_1 H_t(\bm x_1) + \lambda_2 H_t(\bm x_2), \\
		\norm{H_t} := \sup_{\bm x \in \mathcal X, \norm{\bm x}_{\mathcal X}\leq 1}\abs{H_t(\bm x)} < \infty,
		\end{gathered}
	\end{equation}
	where $\norm{H_t}$ denotes the induced functional norm. Furthermore, the norm of a sequence of functionals is defined by $\norm{\bm H}:= \sup_{t \in \mathbb Z}\norm{H_t}$.

	\item  A sequence of functionals 
	$\bm H$ is time-homogeneous if 
	for any $t, \tau \in \mathbb Z$,
	$
		H_t(\bm x) = H_{t+\tau}(\bm x\depen{\tau})
	$
	where $x\depen{\tau}(s) := x(s-\tau)$ for all $s\in \mathbb Z$.
\end{enumerate}
\end{definition}

Following the Definition \ref{def:concept_space}, we consider the concept space
\begin{equation}
	\begin{aligned}
	\mathcal C = & \{ \bm H: \text{ $H_t \in \bm H$ is linear continuous, $\bm H$}\\
	 &~~\text{ is causal and time-homogeneous}
	\}.
\end{aligned}
\end{equation}
Note that $\hrnn \subset \cs$ and $\hcnn\depen{l} \subset \cs$, which follows 
from the fact that any sequences in these two spaces have convolutional (thus linear) representations.
For a given $\bhh$, we denote its corresponding
representation by $\rhh$.
We see that RNNs and CNNs are different in the sense that for RNNs,
the representation $\rho\depen{\bhh}(s) =c^\top W^{s-1}U$,
which has an infinite support in time.
However, for CNNs, the representation $\rhh$ is a finitely supported sequence
with radius $r(\rhh)=l^K-1$.
This can be understood as the maximal memory length of the CNN model, which is usually called the receptive field.

Next we present a key lemma regarding the hypothesis space,
which shows that any target in the concept space has a convolutional representation.
\begin{lemma}\label{lmm:representation}
For any $\bm H \in \cs$,
    there exists a unique $\ell^2$ sequence $\rh :\mathbb N \to \mathbb R^d$ such that 
    \begin{equation}\label{eq: integral representation}
    	H_t(\bm x) = \sum_{s=0}^{\infty} \rho\depen{\bm H}(s)^\top x(t-s), \quad t \in \mathbb Z.
    \end{equation}
    For a target $\bm H$,
    the corresponding representation is denoted as $\rh$.
    The approximation of $\bm H$ by RNNs or CNNs is equivalent to the approximation of $\rh$ using the respective $\rhh$.
\end{lemma}

{

\begin{remark}
The hypothesis space is related with  linear time-invariant system (LTI system).
In fact,
a causal LTI system will induce a linear functional which satisfies Definition \ref{def:concept_space}. However, we note that not every linear functional satisfying
Definition \ref{def:concept_space}
corresponds to an LTI system. 
Our aim is to investigate general funtionals, and does not assume the data is generated from a hidden linear system.
\end{remark}
}

\subsection{Summary of approximation results for RNNs} \label{sec:app_RNN}

In order to facilitate the presentation of CNN results and subsequent comparisons, we first review the main approximation results for RNNs proved in \citet{li2021on}
\footnote{Note that the continuous time setting was considered there, but the main conclusions remain the same under discretization.}.
It was shown that under fairly general conditions, $\hrnn$ is dense in $\cs$.
With further assumptions that $\rh$ decays at least exponentially in time,
an explicit bound of the approximation rate can be obtained. That is, for any $\bm H  \in \cs$, there exists $\bhh  \in \hrnn ^{(m)}$ such that
\begin{equation}
	\norm{\bm H - \bhh} \equiv \sup_{t \in \mathbb R}\norm{H_t - \hat H_t} \leq \frac{C \gamma d}{\beta m^\alpha},
\end{equation}
where $C>0$ is a universal constant, $\beta$ denotes the (exponential) decay rate of $\rh$ and $\gamma$ measures the smoothness of $\bm H$.
This result shows that a target functional family generating the temporal relationship
can be well approximated by RNN models if it is smooth and decays sufficiently fast.
For a target with long memory, say a power law decay, the number of parameters sufficient to approximate it using RNNs may increase exponentially. This is the so-called ``curse of memory'' in approximation of RNNs \citep{li2021on}.


\subsection{Approximation results for CNNs}
In this section, we present the main approximation results of this paper, which can be viewed as parallel results to those discussed in section \ref{sec:app_RNN}, but for CNNs.
This allows us to understand precisely the similarity and differences of CNNs and RNNs with respect to their approximation capabilities, when applied to time series modelling.

First, we prove a density result.
Recall that $\hcnn\depen{l} \subset \cs$.
The density result here shows that $\hcnn\depen{l}$ is in fact dense in $\cs$. 
It is also understood as a universal approximation property of linear CNNs applied to linear functionals.
\begin{theorem}\label{thm:UAPCNN}
	{\normalfont{(UAP for CNNs)}}
	Let $\bm H \in \mathcal C$. Then for any $\epsilon>0$,
	there exist $\bhh \in \hcnn\depen{l}$  such that
	\begin{equation}
		\norm{\bm H - \bhh} \equiv \sup_{t \in \mathbb Z} \norm{H_t - \widehat H_t} < \epsilon.
	\end{equation}
\end{theorem}
Now we discuss some basic understanding of Theorem \ref{thm:UAPCNN}.
For any $\bm x$ such that $\norm{\bm x}_{\mathcal X}\leq 1$, we have 
\begin{align}
    &\abs{H_t(\bm x)  - \hat H_t(\bm x)}^2 \leq 
    \sum_{s=0}^\infty \Big\rvert \rh(s) - \rhh(s) \Big\rvert^2
    \\
    & = \sum_{s=0}^{l^K-1}\Big\rvert \rh(s) - \rhh(s)\Big\rvert^2
    + \sum_{s=l^K}^{\infty} \Big\rvert \rh(s) \Big\rvert ^2, \nonumber
\end{align}
where we use the fact that $r(\rhh) = l^K-1$. 
Consider the approximation error on two intervals.
On $[l^K, \infty]$, $\rhh$ is always zero, 
thus the error only depends on the tail sum of $\rh$. That is, the long term memory of the target.
Since $\rh\in\ell^2$, this error converges to zero when we choose a deep model (with $K$ appropriately large).
On $[0, l^K-1]$, we can show that there exits a $\rhh$ with a sufficient number of channels such that $\rhh = \rh$ in this range.

The error on $[l^K, \infty]$ is easy to analyse as it only involves the decay of $\rh$, which enters in the approximation of both RNNs and CNNs.
The more interesting phenomenon occurs on the interval $[0, l^K-1]$.
A natural question is,
given a target $\bm H$,
how does the approximation error on this interval depend on the number of channels and depth of CNNs?
This question is important since it reveals the structures in the target functional family which facilitate
efficient approximation using CNNs besides the decay of memory. In other words, it characterises theoretically the type of temporal modelling tasks for which CNNs are naturally suited.

Interestingly, it turns out that the approximation rate depends
on the spectrum of $\rh$ under a suitable tensorisation  procedure.
We motivate this idea by simple examples where $d=1$ and $l=K=2$.
\begin{example}
Recall that a target $\bm H$ has the representation $\rh$.
Since $r(\rhh) = 4$ here, the CNN performs approximation only on $[0, 3]$, hence we restrict the target as $\rh_{[0,3]} \in \mathbb R^4$. Denote the rearrangement operator by $T$, and 
\begin{equation}
    T\Big(\rh_{[0,3]}\Big) = 
    \begin{pmatrix}
		\rhn_0 & \rhn_2\\
		\rhn_1 & \rhn_3
	\end{pmatrix}.
\end{equation}
We study the singular value decomposition (SVD) of the matrix $T\Big(\rh_{[0,3]}\Big)$.
If $T\Big(\rh_{[0,3]}\Big)$ has only one non-zero singular value, then $\mathrm{rank}~T\Big(\rh_{[0,3]}\Big) = 1$.
It is straightforward to deduce that a CNN with 1 channel in both layers is sufficient to represent it. 
In fact, let $\bm w_1$ and $\bm w_2$ be the filter on the first and second layer respectively. The resulting representation is
\begin{align}
	\rhh = \bm w_2 \Conv{}_{2} \,
	\bm w_1 &= (w_{11},w_{12})\Conv{}_{2}(w_{21},w_{22})\\
	&=(w_{11}w_{21},w_{12}w_{21},w_{11}w_{22},w_{12}w_{22}) ,\nonumber \\
	T(\rhh) & = \begin{pmatrix}
		w_{11}w_{21} & w_{11}w_{22}\\
		w_{12}w_{21} & w_{12}w_{22}
	\end{pmatrix} \\
	& = 
	\begin{pmatrix}
	    w_{11} \\
		w_{12}
	\end{pmatrix}
	\begin{pmatrix}
		w_{21} &
		w_{22}
	\end{pmatrix}.
\end{align}
This can represent any 2 by 2 rank 1 matrix.
Thus, any $\rh$ such that $T\big(\rh_{[0,3]}\big)$ with rank no more than 1 can be represented.

In order to represent a rank 2 target, 
we need to increase the number of channels.
Let the first layer have $2$ channels.
Then the resulting representation is
\begin{equation}
	\begin{gathered}
		\rhh = \bm w_2 \Conv{}_{2} \,  \bm w_1 + \bm v_2 \Conv{}_{2} \,  \bm v_1, \\
	T(\rhh) = \begin{pmatrix}
		w_{11}w_{21} & w_{11}w_{22}\\
		w_{12}w_{21} & w_{12}w_{22}
	\end{pmatrix} +
	 \begin{pmatrix}
		v_{11}v_{21} & v_{11}v_{22}\\
		v_{12}v_{21} & v_{12}v_{22}
	\end{pmatrix} \\
	 = 
	\begin{pmatrix}
		\frac{1}{\sqrt\sigma_1}w_{11} & \frac{1}{\sqrt\sigma_2}v_{11} \\
		\frac{1}{\sqrt\sigma_1}w_{12} & \frac{1}{\sqrt\sigma_2}v_{12}
	\end{pmatrix}
	\begin{pmatrix}
		\sigma_1 & \\
		& \sigma_2
	\end{pmatrix}
	\begin{pmatrix}
		\frac{1}{\sqrt\sigma_1}w_{21} & \frac{1}{\sqrt\sigma_2}v_{21} \\
		\frac{1}{\sqrt\sigma_1}w_{22} & \frac{1}{\sqrt\sigma_2}v_{22}
	\end{pmatrix}^\top.
	\end{gathered}
\end{equation}
 That is, $T(\rhh)$ can represent the SVD for any 2 by 2 matrix with appropriate choices of $\bm w_1,\bm w_2$ and $\bm v_1$, which implies that $T\big(\rh_{[0,3]}\big)$ with rank no more than 2 can be represented.

The above two examples show that the number of channels 
needed for an exact representation ($\rhh_{[0,3]} = \rh_{[0,3]}$) depends
on the rank of targets, i.e. the number of non-zero singular values of $T\big(\rh_{[0,3]}\big)$.
Furthermore, when there is no such exact representations, we need to decide the approximation error.
Suppose the target is with rank 2,
but we use a CNN with only $1$ channel to approximate it.
By the Eckart–Young–Mirsky theorem, 
the best approximation error is equal to the smallest singular value of 
$T\big(\rh_{[0,3]}\big)$.
This inspires us to relate the approximation error with the decay of singular values.
\end{example}

The above examples illustrate the basic approach
on how to determine the number of channels needed to achieve a given approximation accuracy.
Although the main tool used in above examples, SVD, is only valid for $K=2$ where we can reshape the vector representation of $\rh$ into a matrix, in general we can extend the above techniques to any $K$
using higher order singular value decomposition (HOSVD)
\citep{DeLathauwer2000}. Now we introduce some basic facts about HOSVD.

\paragraph{HOSVD basics.}
For a  discrete sequence $\bm \rho$,
we denote the tensorisation of $\bm \rho_{[0,l^K-1]}$ by $T_{l^K}(\bm \rho)$.
This is an order $K$ tensor,
with all the dimensions equal to $l$ i.e.
$T_{l^K}(\bm \rho) \in \mathbb R^{l \times l \times \cdots \times l}$.
The tensorisation follows column major ordering.
The singular values of $T_{l^K}(\bm \rho)$ is defined by applying the usual  SVD for matrices to the $K$ matrices with size $\mathbb R^{l\times l^{K-1}}$ obtained by mode-$k$ flattening of the tensor.
This gives rise to at most $lK$ singular values.
Define the rank of the tensor by its number of non-zero singular values.
We defer the precise but involved definition of HOSVD to the appendix.

Similar to the SVD low rank approximation for matrices, we have the following error estimation for tensors \citep{DeLathauwer2000}.

\begin{lemma}\label{lmm:lowrankapp_tensor}
Denote the singular values of $\tensor A \in \mathbb R^{l \times l \times \cdots \times l}$ by $\sigma_1\ge\sigma_2\ge\cdots\ge\sigma_{lK}\geq 0$,
we have 
\begin{equation} 
	\inf_{\bhh} \left\|\tensor A - \ten{\rhh}\right\|
	\leq
	\left ( \sum_{i=K'+1}^{\rank{\tensor A}}
	\sigma_i^2 \right)^{\frac 1 2} ,
\end{equation}
where the infimum is taken over all $\bhh \in \hcnn^{(l, K,\set{M_k})}$  such that $K' < \rank \tensor A$, where
$K' = \rank{\ten{\rhh}}$. 
\end{lemma}

Based on Lemma \ref{lmm:lowrankapp_tensor},
one can measure the complexity of a target $\bm H$
by the decay rate of singular values of
$\ten{\rh}$.

\paragraph{Complexity measure and approximation rates.}

Now we can get down to define appropriate complexity measures and prove approximation rates of CNNs.
Based on previous discussions on HOSVD, we have the following key definition.

\begin{definition}\label{def:Hnorm}
	Consider a sequence of functionals 
	$\bm H$ with associated representation 
	$\rh$.
	For any $l,K\in\mathbb N_+$, let 
	$\sigma_1\depen{K}\ge\sigma_2\depen{K}\ge\cdots\ge\sigma_{lK}\depen{K}\geq 0$ 
	be the singular values of 
	$\ten{\rh}$.
	Let $g \in \bm c_0(\mathbb N, \mathbb R_{+})$ be a non-increasing function with zero limit at infinity.
Define the complexity measure of $\bm H$ by
\begin{equation}\label{eq: Hnorm}
\begin{aligned} 
	C\depen{l,g}(\bm H) = \inf \Bigg\{c: \left(\sum_{i=s+K}^{lK}\abs{\sigma_i\depen{K}}^2\right)^{\frac 1 2} \leq c g(s), \\ ~s\geq 0, K\ge 1\Bigg\}.
\end{aligned}
\end{equation}
\end{definition}
Based on this complexity measure, we define a concept space with certain spectral regularity (measured by the decay of singular values)
\begin{equation}
	\cs\depen{l,g}:=\set{\bm H \in \mathcal C: C\depen{l,g}(\bm H)<\infty}.
\end{equation}

Next, we discuss some facts and examples to enhance the understanding of the space $\cs\depen{l,g}$.
\begin{remark}
Suppose the function $g$ is monotonously decreasing and strictly positive.
Then for any $\bm H\in\cs$ such that $\rh$ is finitely supported, 
we have $\bm H \in \cs\depen{l,g}$.
\end{remark} 
\begin{remark}
Suppose $\bm H\in\cs$ with $\rh$ finitely supported.
Then there exists a finitely supported decreasing $g$ such
that $\bm H \in \cs\depen{l,g}$.
\end{remark}
The details can be found in appendix.

The reason why the summation in \eqref{eq: Hnorm} starts at $s+K$ is that,
when $K$ increases, there will be additional singular values resulting from the additional matrix SVDs.
These singular values are all equal to 
the 2-norm of $\rh$.
We show this by the following example.

\begin{example}
Suppose the filter size $l=2$
and $\rh =$ 
$(1,0,0,1,0,0,0,\cdots)$.
The following table shows the corresponding singular values of $\ten{\rh}$ for different $K$:
\begin{center}
\begin{tabular}{ c r }
 K & Singular Values of $\ten{\rh}$ \\ 
 1 & $(1,1)$  \\  
 2 & $(1,1,1,1)$ \\
 3 & $(\sqrt{2},1,1,1,1,0)$\\
 4 & $(\sqrt{2},\sqrt{2},1,1,1,1,0,0)$
\end{tabular}
\end{center}
When $K\geq 2$, $[0,l^K-1]$ covers the support of $\rh$.
If we further increase $K$, 
there will be additional singular values all equal to $\norm{\rh}_2 = \sqrt{2}$.

For $K \geq 2$, we have
\begin{equation}\label{eq: sigular example}
    \sum_{i=s+K}^{lK}\abs{\sigma_i\depen{K}}^2 = 
            \begin{cases}
                2, &  s = 1 \\
                1,      &   s =2.\\
                0,    & \text{otherwise}
            \end{cases}
\end{equation}
Then given $g$ one can compute $C\depen{l,g}(\bm H)$ 
based on \eqref{eq: Hnorm} and \eqref{eq: sigular example}.
\end{example}

Based on Definition \ref{def:Hnorm}, we can now present our main result on the approximation rate of CNNs.

\begin{theorem}{\normalfont{(Approximation rate for CNNs)}}\label{thm: error bound}
Fix $l\geq 2$ and $g \in \bm c_0(\mathbb N, \mathbb R_{+})$.
For any $\bm H  \in \cs\depen{l,g}$ and any set of parameters $(K,\set{M_k})$,
we have
\begin{equation}
	\begin{aligned}
		 \frac 1 {\sqrt d} \sup_{t \in [l^K,\infty]} \abs{\rh(t)}
		\leq 
		 \inf_{\bhh \in \hcnn\depen{l,K,\set{M_k}}}\norm{\bm H - \bhh} 
		 \leq \\
		 d \ g(KM^{\frac 1 K}-  K)  C\depen{l,g}(\bm H)
		   + { \norm{\rh_{[l^K,\infty]}}_2},
	\end{aligned}
\end{equation}
where $M :=  \frac 1 d(\sum_{k=2}^K M_k M_{k-1} -lK)$ denotes the effective number of filters.
\end{theorem}
As discussed earlier,
the term $\norm{\rh_{[l^K,\infty]}}_2$
is the error on $[l^K, \infty]$, due to the limited support of depth-$K$ CNNs.
This error can be reduced by increasing the number of layers $K$,
and it is less important as $l^K$ increases exponentially fast.

Hence, the subsequent discussion will focus on the more interesting term, i.e. the approximation error on the interval $[0,l^K-1]$: $d \ g(KM^{\frac 1 K}-  K)  C\depen{l,g}(\bm H)$.
This is related to the complexity measure of $\bm H$, which is small if there exists effective low rank structures in the target.
With the fact that $\rank{\ten{\rhh}}$ is at least 
$KM^{\frac 1 K}$ (see details in appendix),
combining Lemma \ref{lmm:lowrankapp_tensor} and \eqref{eq: Hnorm} gives the result.
Given a target $\bm H$,
this error can be reduced by either increasing the number of filters
$M$ or 
the number of layers
$K$.

Whether a target $\bm H$ can be easily approximated 
depends on $C\depen{l,g}(\bm H)$,
which is determined by the decay of singular values of $\ten{\rh}$.
For $\ten{\rh}$ with fast decaying singular values,
even if the rank is large or $\rh$ decays slowly,
one can still have a good approximation with CNNs. This is very different from previous results of RNNs.

We now introduce an example to illustrate this observation.
Consider 3 targets with the following representations
\begin{equation*}
    \begin{gathered}
        \rho_1(t)=\begin{cases}
        \frac {\pi^2} {12}, &  t = 17,18,25,26 \\
        0,     & \text{otherwise}
        \end{cases}
     ,  \\
    \rho_2(t)=\begin{cases}
        \frac {\pi^2} {12},  & t = 9,15,19,26 \\
        0,      & \text{otherwise}
    \end{cases}, \ \  \rho_3(t) = \frac 1 t.
    \end{gathered}
\end{equation*}
Note that all of them are instances of the adding problem. Moreover, 
they all have the same magnitude,
$\norm{\bm \rho_1}_2=\norm{\bm \rho_2}_2=\norm{\bm \rho_3}_2$.
The support of both $\bm \rho_1, \bm\rho_2$ are finite, 
while $\bm \rho_3$ have an
infinite support.
Furthermore, $\bm \rho_1$ and $\bm \rho_2$
have the same number of non-zero entries.
We plot the error bound of each targets in 
\figref{fig:error}.

First, from the figure we conclude that
the approximation can be reduced by 
increasing $K$ or $M$.
For a fixed $K$, 
when $M$ is sufficiently large such that the rank of the model is no less than the rank of the target, 
the error will no longer decrease if we further increase 
$M$. At this point
the error only comes from the tail which 
can only be reduced by increasing $K$ (hence increasing the support, or receptive field of the CNN model).

Next, we compare $\bm{\rho_1}$ and $\bm{\rho_2}$.
From \figref{fig:error}, we conclude that 
$\bm{\rho_1}$ is easier to approximate than $\bm{\rho_2}$.
The only difference between these two targets is 
the position of their non-zero entries,
which results in $T_{2^5}(\bm{\rho_1})$ being rank 5 while 
$T_{2^5}(\bm{\rho_2})$ being rank 10.
This is the case where low rank structures facilitates the approximation by CNNs.

Finally, we make a comparison between $\bm{\rho_2}$  and $\bm{\rho_3}$.
We conclude that $\bm{\rho_3}$ is easier to approximate
as the overall error is smaller than $\bm{\rho_2}$, in spite of the fact that $\bm{\rho_2}$ is finitely supported and 
has a simple form.
The reason is that the 
tail sum of
singular values of $\ten{\bm{\rho_3}}$ decays faster than that of $\ten{\bm \rho_2}$.
This illustrates the case where a target with fast decaying singular values (i.e. low effective rank) can
be easily approximated.

\begin{figure*}[t]
\centering
\subfigure[$\bm \rho_1$]{
{\epsfig{file = 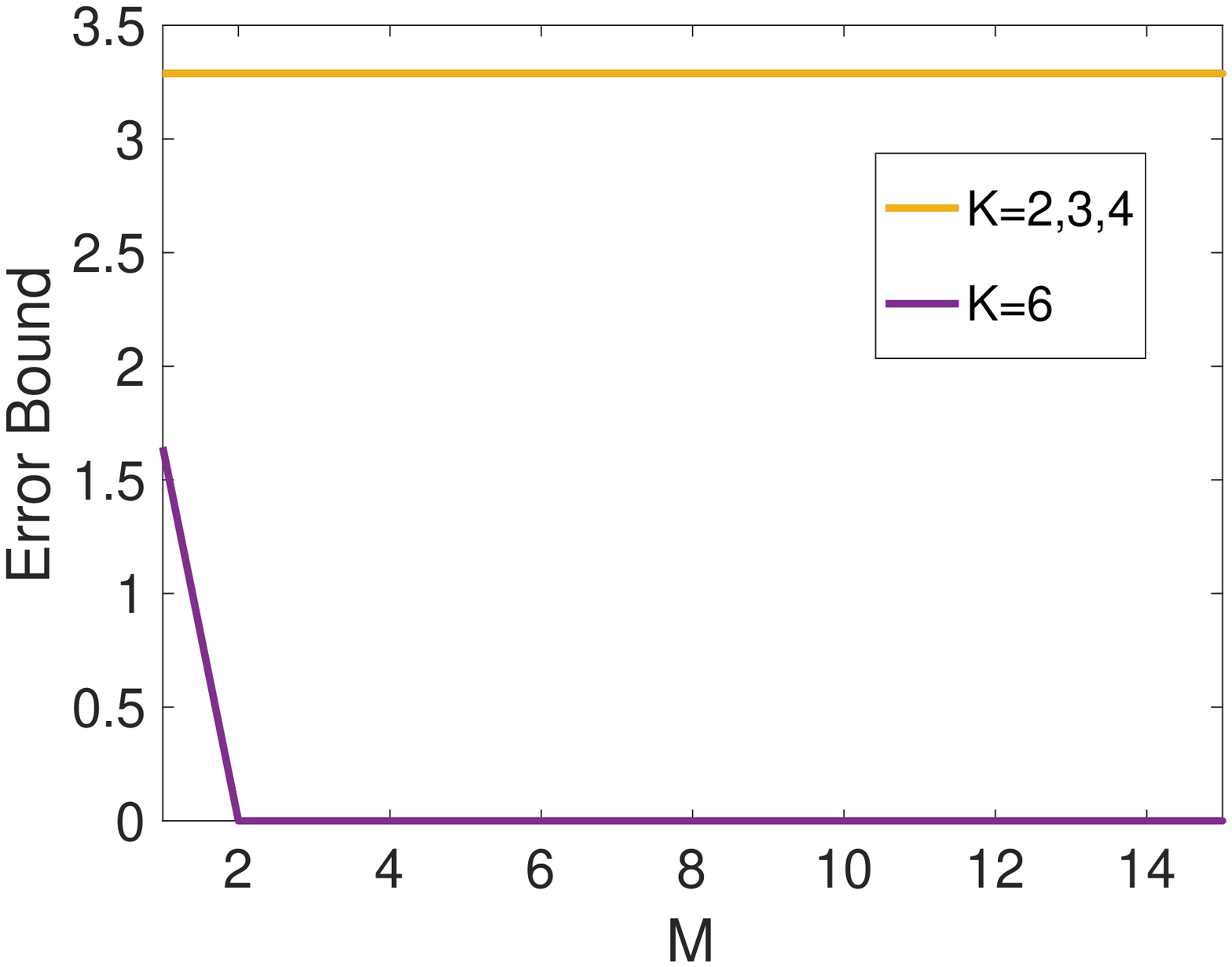 , width = 5.5cm}}
\label{fig:rank5}%
}
\subfigure[$\bm \rho_2$]{%
{\epsfig{file = 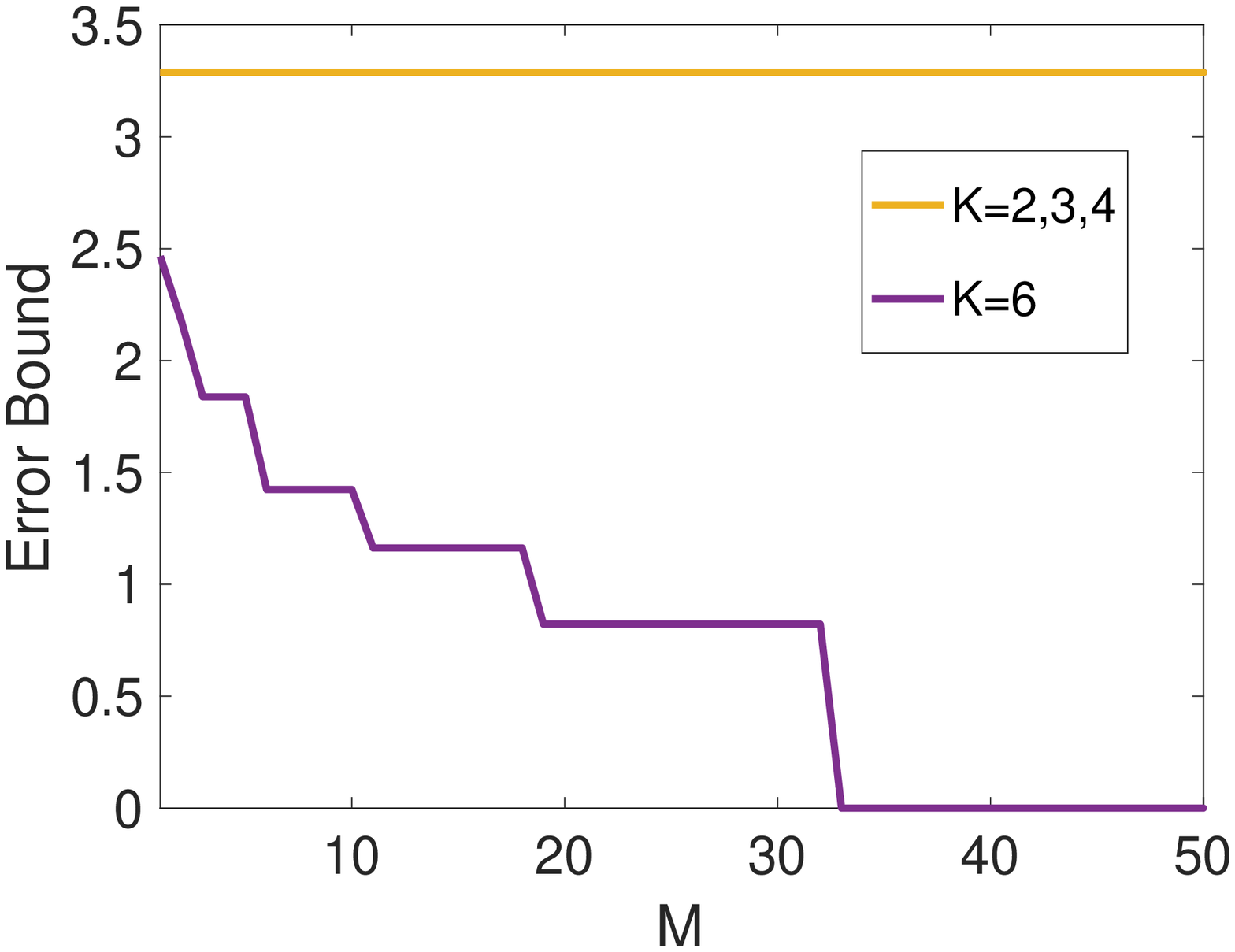, width = 5.6cm}}
\label{fig:rank10}%
}
\subfigure[$\bm \rho_3$]{%
{\epsfig{file = 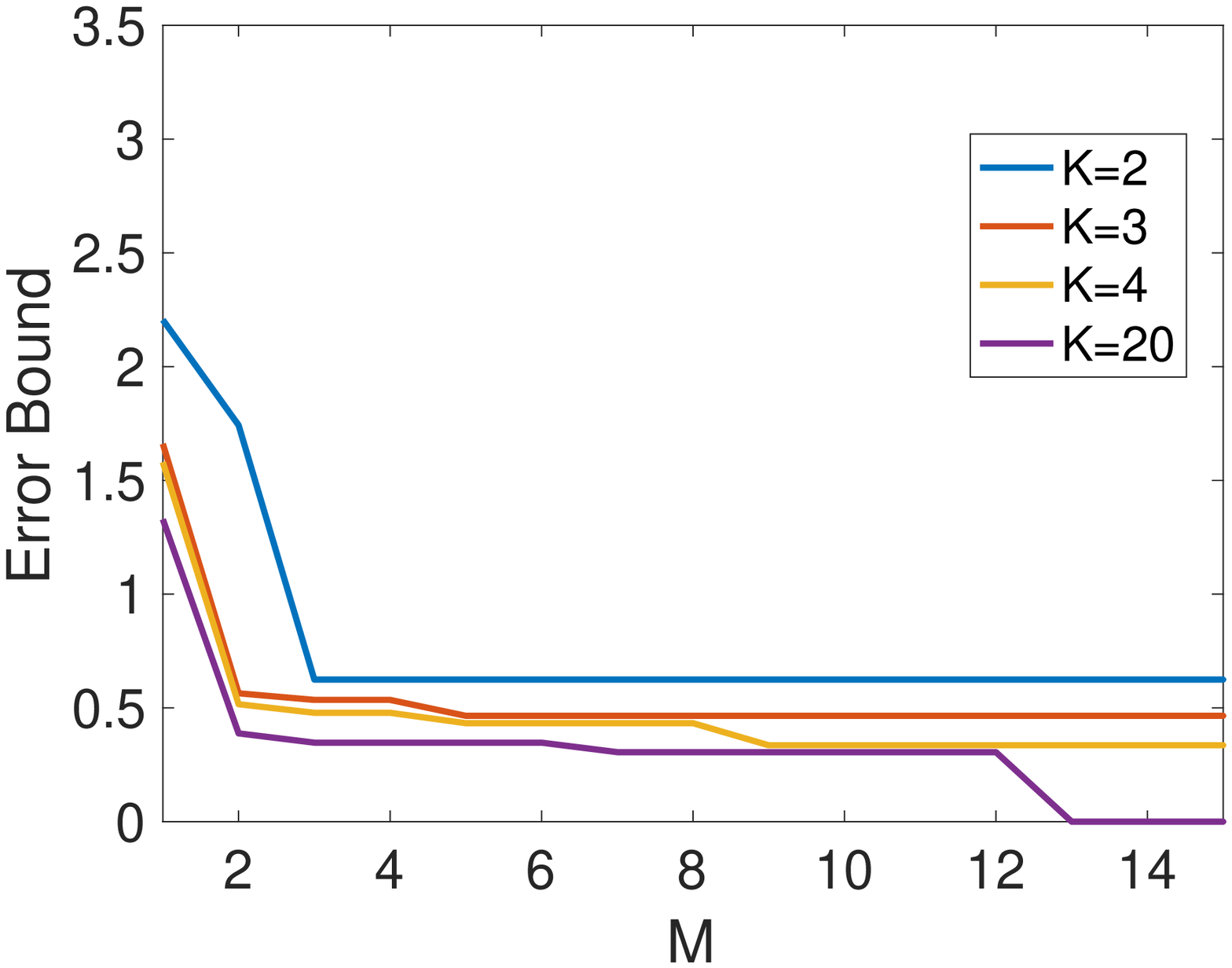, width = 5.5cm}}
\label{fig:infinite}%
}
\caption{
Let $l=2$.
We plot the upper bound of the approximation errors against the effective number of filters.
Each line corresponds to a fixed $K$. 
The error bound is calculated based on Theorem \ref{thm: error bound}, 
where the tail term is calculated directly and the rank term is determined by numerically compute the singular values of 
$\ten{\bm \rho_i}$.
These curves describe, according to Theorem \ref{thm: error bound}, the decay of approximation errors of CNNs as one increases the number of filters. Observe that the decay depends on the low rank structures of the targets.
}
\label{fig:error}
\end{figure*}

\section{Comparison of CNNs and RNNs in terms of approximation}
\label{sec:comparison}



As motivated earlier, how to judiciously choose architectures for time series modelling is an important practical problem, 
and prior works only address it empirically.
Combining the approximation theory for CNNs developed here and those for RNNs in \citet{li2021on},
we are in a position to provide theoretical answers to such problems, 
albeit in specific settings. This is the purpose of the present section.


In short, our theoretical analyses show that neither CNNs
nor RNNs are always better than the other.
Whether one out-performs its counterpart depends on the properties of the target relationship.
Our results make this statement precise, and provide general guidelines 
on how to select architectures based on applications.
We begin by presenting two representative examples for the two aforementioned cases,
from which we can infer key differences between RNNs and CNNs for time series modelling.

\paragraph{Example where RNNs out-perform CNNs.}
We assume a scalar input with $d=1$.
Consider a target $\bm H\in \cs$ 
with the representation $\rho^{\bm H}(t) = \gamma^t$, where $ 0<\gamma<1.$
It is easy for RNNs to approximate this target, since the representation has a power form. In fact, we have $\bm H \in \hrnn^{(1)}$, i.e. a RNN with one hidden unit is sufficient to achieve an exact representation,
with 0 approximation error for any $\gamma$.

For a CNN $\bhh \in \hcnn\depen{l}$, based on the lower bound of the approximation error from Theorem \ref{thm: error bound}, we have that
$
	\norm{\bm H - \bhh}^2 \geq  \sup_{t \in [l^K,\infty]} \norm{\rh(t)}_2.
$
Thus, in order to achieve an approximation error with $\norm{\bm H - \bhh}^2 < \epsilon$, we have
$
	 \sup_{t \in [l^K,\infty]} \norm{\rh(t)}_2 =  \gamma^{l^K}
	< \epsilon.
$
This implies $l^K \geq \frac{\log ( \epsilon )}{\log (\gamma )}$.
That is to say, the  necessary number of layers to 
achieve an approximation error smaller than $\epsilon$ 
diverges to infinity as $\gamma$ approaches 1.

\paragraph{Example where CNNs out-perform RNNs.}
Still assume a scalar input with $d=1$.
Consider a target $\bm H \in \cs$ with representation
$
    \rho^{(\bm H)}(t) = \delta(t-2^K)
$
where  $K \in \mathbb N_+$. 
This is an discrete time impulse at $t = 2^K$ and can be regarded
 as the copying memory problem \citep{Bai2018}:
 given an input $\bm x$, the output $\bm y$ satisfies $y(t) = x(t - 2^K)$. 
We have $\bm H \in \hcnn^{(2,K,\set{1})}$.
Thus, a $K$-layer CNN with one channel per layer
is sufficient to achieve an exact representation. 

Recall that RNN approximates the target $\rh$ with a power sum $\rhh(s) = c^\top W^{s-1}U$. 
Suppose here $W \in \mathbb R^{m \times m}$ is a diagonalisable matrix with negative eigenvalues.
The latter has a special structure, which makes it difficult to approximate certain functions.
Suppose $u$ is a $m$-term power sum, i.e.
$
	u(t) = c_0 + \sum_{i=1}^m c_i\  \gamma_i^{ t}.
$
Then according to \citep{Borwein1996}, we have
\begin{equation}\label{eq: exp property}
	\abs{u(t+1)-u(t)} \leq \frac{2m}{t} \sup_{s\geq 0} u(s).
\end{equation}
For a fixed number of terms $m$,
the changes between $u(t+1)$ and $u(t) $ approaches zero
as $t$ approaches infinity.
This means that if there is a sudden change in $u$ far from origin, the number of terms $n$ must be large.

For a RNN $\bhh \in \hrnn\depen{m}$, assume 
$
	\norm{\bm H - \bhh}  < \epsilon.
$
Let $u(s) = c^\top W^{s-1}U$ denote the power sum corresponding to $\bhh$.
Since $W$ is a $m \times m$ matrix,
$u(s)$ is a $m^2$-term power sum.
We have that
$
	\abs{u(2^K+1)-u(2^K)} > 1-2 \epsilon,
$
and hence \eqref{eq: exp property} implies
$
	m^2 > 2^{K-1}\frac{ 1-2\epsilon}{1+\epsilon}.
$
As $k$ increases,
the number of parameters needed for RNNs
to achieve an error less than $\epsilon$
increases exponentially, 
while this increment is linear for CNNs.

For these examples,
the targets actually come from the other respective hypothesis spaces.
This means that the functionals these two hypothesis spaces can represent are quite different.
The difference between RNNs and CNNs comes from 
their underlying structure:
the RNN uses a power sum to approximate the target, while the
CNN uses a finitely supported function.
Therefore,
RNNs are good at approximating targets with exponentially decaying structures, but it is not efficient to handle targets with sudden changes. On the other hand, CNN works well for targets with low ranks or fast decaying singular values. However, it can become inefficient if the tail error term is significant, and at the same time, the truncated term does not possess low rank structures.

{
In practice, the property of decay and sparsity can be checked by querying the target relationship with proper inputs. 
However, for a more general situation when there is only access to sampled data for input/output, inferring these underlying properties of targets remains challenging.

}





\section{Conclusion}

In this paper,
we studied the approximation properties of convolutional architectures when applied to time series modelling.
We considered the simple but representative linear setting.
The approximation error is characterised by both the memory and the spectrum of the target relationship based on a tensorisation argument.
We concluded that a target with a low rank property or fast decaying spectrum can be efficiently approximated by a CNN model.


In particular, these results for CNNs, together with previous results for RNNs, provide basic theoretical insights into the problem of architectural choice for time series modelling: the RNN exploits exponential decaying structures, whereas the CNN exploits low-rank structures.
This forms the first basic step towards a concrete understanding of the interplay between different neural network architectures and the structures of temporal relationships under investigation.

\section*{Acknowledgements}
We would like to thank the anonymous reviewers for their
constructive comments.
HJ is supported by Nationa University of Singapore under PGF scholarship.
QL is supported by the National Research Foundation, Singapore, under
the NRF fellowship (NRF-NRFF13-2021-0005).

\bibliography{reference}

\begin{thebibliography}{31}
\providecommand{\natexlab}[1]{#1}
\providecommand{\url}[1]{\texttt{#1}}
\expandafter\ifx\csname urlstyle\endcsname\relax
  \providecommand{\doi}[1]{doi: #1}\else
  \providecommand{\doi}{doi: \begingroup \urlstyle{rm}\Url}\fi

\bibitem[Bai et~al.(2018)Bai, Kolter, and Koltun]{Bai2018}
Bai, S., Kolter, J.~Z., and Koltun, V.
\newblock {An Empirical Evaluation of Generic Convolutional and Recurrent
  Networks for Sequence Modeling}.
\newblock 2018.
\newblock URL \url{http://arxiv.org/abs/1803.01271}.

\bibitem[Banerjee et~al.(2019)Banerjee, Ling, Chen, Hasan, Langlotz,
  Moradzadeh, Chapman, Amrhein, Mong, Rubin, Farri, and
  Lungren]{BANERJEE201979}
Banerjee, I., Ling, Y., Chen, M.~C., Hasan, S.~A., Langlotz, C.~P., Moradzadeh,
  N., Chapman, B., Amrhein, T., Mong, D., Rubin, D.~L., Farri, O., and Lungren,
  M.~P.
\newblock Comparative effectiveness of convolutional neural network (cnn) and
  recurrent neural network (rnn) architectures for radiology text report
  classification.
\newblock \emph{Artificial Intelligence in Medicine}, 97:\penalty0 79 -- 88,
  2019.
\newblock ISSN 0933-3657.
\newblock \doi{https://doi.org/10.1016/j.artmed.2018.11.004}.
\newblock URL
  \url{http://www.sciencedirect.com/science/article/pii/S0933365717306255}.

\bibitem[Bao et~al.(2019)Bao, Li, Shen, Tai, Wu, and
  Xiang]{Bao2019ApproximationAO}
Bao, C., Li, Q., Shen, Z., Tai, C., Wu, L., and Xiang, X.
\newblock Approximation analysis of convolutional neural networks a**.
\newblock 2019.

\bibitem[Boilard et~al.(2019)Boilard, Gournay, and Lefebvre]{boilard2019a}
Boilard, J., Gournay, P., and Lefebvre, R.
\newblock A literature review of wavenet: theory, application, and
  optimization.
\newblock \emph{Journal of the Audio Engineering Society}, march 2019.

\bibitem[Borwein \& Erd{\'{e}}lyi(1996)Borwein and Erd{\'{e}}lyi]{Borwein1996}
Borwein, P. and Erd{\'{e}}lyi, T.
\newblock {A sharp Bernstein-type inequality for exponential sums}.
\newblock \emph{Journal fur die Reine und Angewandte Mathematik}, 476\penalty0
  (476):\penalty0 127--141, 1996.
\newblock ISSN 00754102.
\newblock \doi{10.1515/crll.1996.476.127}.

\bibitem[Bramwell \& Kreyszig(1979)Bramwell and Kreyszig]{Bramwell1979}
Bramwell, M.~C. and Kreyszig, E.
\newblock {Introductory Functional Analysis with Applications}.
\newblock \emph{The Mathematical Gazette}, 1979.
\newblock ISSN 00255572.
\newblock \doi{10.2307/3616033}.

\bibitem[Chow \& {Xiao-Dong Li}(2000)Chow and {Xiao-Dong Li}]{841860}
Chow, T. W.~S. and {Xiao-Dong Li}.
\newblock Modeling of continuous time dynamical systems with input by recurrent
  neural networks.
\newblock \emph{IEEE Transactions on Circuits and Systems I: Fundamental Theory
  and Applications}, 47\penalty0 (4):\penalty0 575--578, 2000.

\bibitem[{De Lathauwer} et~al.(2000){De Lathauwer}, {De Moor}, and
  Vandewalle]{DeLathauwer2000}
{De Lathauwer}, L., {De Moor}, B., and Vandewalle, J.
\newblock {A multilinear singular value decomposition}.
\newblock \emph{SIAM Journal on Matrix Analysis and Applications}, 21\penalty0
  (4):\penalty0 1253--1278, 2000.
\newblock ISSN 08954798.
\newblock \doi{10.1137/S0895479896305696}.

\bibitem[Doya(1993)]{doya1993universality}
Doya, K.
\newblock Universality of fully connected recurrent neural networks.
\newblock \emph{Dept. of Biology, UCSD, Tech. Rep}, 1993.

\bibitem[Funahashi \& Nakamura(1993)Funahashi and Nakamura]{FUNAHASHI1993801}
Funahashi, K. and Nakamura, Y.
\newblock Approximation of dynamical systems by continuous time recurrent
  neural networks.
\newblock \emph{Neural Networks}, 6\penalty0 (6):\penalty0 801 -- 806, 1993.
\newblock ISSN 0893-6080.

\bibitem[Hochreiter \& Schmidhuber(1997)Hochreiter and
  Schmidhuber]{doi:10.1162/neco.1997.9.8.1735}
Hochreiter, S. and Schmidhuber, J.
\newblock Long short-term memory.
\newblock \emph{Neural Computation}, 9\penalty0 (8):\penalty0 1735--1780, 1997.
\newblock \doi{10.1162/neco.1997.9.8.1735}.
\newblock URL \url{https://doi.org/10.1162/neco.1997.9.8.1735}.

\bibitem[{Jin} et~al.(2018){Jin}, {Finkelstein}, {Mysore}, and {Lu}]{8462431}
{Jin}, Z., {Finkelstein}, A., {Mysore}, G.~J., and {Lu}, J.
\newblock Fftnet: A real-time speaker-dependent neural vocoder.
\newblock In \emph{2018 IEEE International Conference on Acoustics, Speech and
  Signal Processing (ICASSP)}, pp.\  2251--2255, 2018.
\newblock \doi{10.1109/ICASSP.2018.8462431}.

\bibitem[Kalchbrenner et~al.(2018)Kalchbrenner, Elsen, Simonyan, Noury,
  Casagrande, Lockhart, Stimberg, van~den Oord, Dieleman, and
  Kavukcuoglu]{pmlr-v80-kalchbrenner18a}
Kalchbrenner, N., Elsen, E., Simonyan, K., Noury, S., Casagrande, N., Lockhart,
  E., Stimberg, F., van~den Oord, A., Dieleman, S., and Kavukcuoglu, K.
\newblock Efficient neural audio synthesis.
\newblock In Dy, J. and Krause, A. (eds.), \emph{Proceedings of the 35th
  International Conference on Machine Learning}, volume~80 of \emph{Proceedings
  of Machine Learning Research}, pp.\  2410--2419, Stockholm Sweden, 10--15 Jul
  2018. PMLR.
\newblock URL \url{http://proceedings.mlr.press/v80/kalchbrenner18a.html}.

\bibitem[Kolda \& Bader(2009)Kolda and Bader]{Kolda2009}
Kolda, T.~G. and Bader, B.~W.
\newblock {Tensor decompositions and applications}, 2009.
\newblock ISSN 00361445.

\bibitem[Li et~al.(2005)Li, Ho, and Chow]{article}
Li, X.-D., Ho, J. K.~L., and Chow, T. W.~S.
\newblock Approximation of dynamical time-variant systems by continuous-time
  recurrent neural networks.
\newblock \emph{IEEE Transactions on Circuits and Systems II Analog and Digital
  Signal Processing}, 52:\penalty0 656--660, 10 2005.

\bibitem[Li et~al.(2021)Li, Han, E, and Li]{li2021on}
Li, Z., Han, J., E, W., and Li, Q.
\newblock On the curse of memory in recurrent neural networks: Approximation
  and optimization analysis.
\newblock In \emph{International Conference on Learning Representations}, 2021.
\newblock URL \url{https://openreview.net/forum?id=8Sqhl-nF50}.

\bibitem[Maass et~al.(2007)Maass, Joshi, and Sontag]{maass2007computational}
Maass, W., Joshi, P., and Sontag, E.~D.
\newblock Computational aspects of feedback in neural circuits.
\newblock \emph{PLOS Computational Biology}, 3\penalty0 (1):\penalty0 e165,
  2007.

\bibitem[Matthews(1993)]{Matthews1993ApproximatingNF}
Matthews, M.~B.
\newblock Approximating nonlinear fading-memory operators using neural network
  models.
\newblock \emph{Circuits, Systems and Signal Processing}, 12:\penalty0
  279--307, 1993.

\bibitem[Nakamura \& Nakagawa(2009)Nakamura and
  Nakagawa]{10.1007/978-3-642-04277-5_60}
Nakamura, Y. and Nakagawa, M.
\newblock Approximation capability of continuous time recurrent neural networks
  for non-autonomous dynamical systems.
\newblock In \emph{Artificial Neural Networks -- ICANN 2009}, pp.\  593--602,
  2009.

\bibitem[Oono \& Suzuki(2019)Oono and Suzuki]{pmlr-v97-oono19a}
Oono, K. and Suzuki, T.
\newblock Approximation and non-parametric estimation of {R}es{N}et-type
  convolutional neural networks.
\newblock In Chaudhuri, K. and Salakhutdinov, R. (eds.), \emph{Proceedings of
  the 36th International Conference on Machine Learning}, volume~97 of
  \emph{Proceedings of Machine Learning Research}, pp.\  4922--4931. PMLR,
  09--15 Jun 2019.
\newblock URL \url{http://proceedings.mlr.press/v97/oono19a.html}.

\bibitem[Oord et~al.(2016)Oord, Dieleman, Zen, Simonyan, Vinyals, Graves,
  Kalchbrenner, Senior, and Kavukcuoglu]{oord2016wavenet}
Oord, A. v.~d., Dieleman, S., Zen, H., Simonyan, K., Vinyals, O., Graves, A.,
  Kalchbrenner, N., Senior, A., and Kavukcuoglu, K.
\newblock Wavenet: A generative model for raw audio.
\newblock \emph{arXiv preprint arXiv:1609.03499}, 2016.

\bibitem[Paine et~al.(2016)Paine, Khorrami, Chang, Zhang, Ramachandran,
  Hasegawa{-}Johnson, and Huang]{DBLP:journals/corr/PaineKCZRHH16}
Paine, T.~L., Khorrami, P., Chang, S., Zhang, Y., Ramachandran, P.,
  Hasegawa{-}Johnson, M.~A., and Huang, T.~S.
\newblock Fast wavenet generation algorithm.
\newblock \emph{CoRR}, abs/1611.09482, 2016.
\newblock URL \url{http://arxiv.org/abs/1611.09482}.

\bibitem[{Prenger} et~al.(2019){Prenger}, {Valle}, and {Catanzaro}]{8683143}
{Prenger}, R., {Valle}, R., and {Catanzaro}, B.
\newblock Waveglow: A flow-based generative network for speech synthesis.
\newblock In \emph{ICASSP 2019 - 2019 IEEE International Conference on
  Acoustics, Speech and Signal Processing (ICASSP)}, pp.\  3617--3621, 2019.
\newblock \doi{10.1109/ICASSP.2019.8683143}.

\bibitem[Sch{\"a}fer \& Zimmermann(2006)Sch{\"a}fer and
  Zimmermann]{10.1007/11840817_66}
Sch{\"a}fer, A.~M. and Zimmermann, H.~G.
\newblock Recurrent neural networks are universal approximators.
\newblock In \emph{Artificial Neural Networks -- ICANN 2006}, pp.\  632--640,
  2006.

\bibitem[Sch{\"a}fer \& Zimmermann(2007)Sch{\"a}fer and
  Zimmermann]{schafer2007recurrent}
Sch{\"a}fer, A.~M. and Zimmermann, H.-G.
\newblock Recurrent neural networks are universal approximators.
\newblock \emph{International journal of neural systems}, 17\penalty0
  (04):\penalty0 253--263, 2007.

\bibitem[van~den Oord et~al.(2016)van~den Oord, Dieleman, Zen, Simonyan,
  Vinyals, Graves, Kalchbrenner, Senior, and Kavukcuoglu]{Oord2016}
van~den Oord, A., Dieleman, S., Zen, H., Simonyan, K., Vinyals, O., Graves, A.,
  Kalchbrenner, N., Senior, A., and Kavukcuoglu, K.
\newblock {WaveNet: A Generative Model for Raw Audio}.
\newblock pp.\  1--15, 2016.
\newblock URL \url{http://arxiv.org/abs/1609.03499}.

\bibitem[van~den Oord et~al.(2018)van~den Oord, Li, Babuschkin, Simonyan,
  Vinyals, Kavukcuoglu, van~den Driessche, Lockhart, Cobo, Stimberg,
  Casagrande, Grewe, Noury, Dieleman, Elsen, Kalchbrenner, Zen, Graves, King,
  Walters, Belov, and Hassabis]{pmlr-v80-oord18a}
van~den Oord, A., Li, Y., Babuschkin, I., Simonyan, K., Vinyals, O.,
  Kavukcuoglu, K., van~den Driessche, G., Lockhart, E., Cobo, L., Stimberg, F.,
  Casagrande, N., Grewe, D., Noury, S., Dieleman, S., Elsen, E., Kalchbrenner,
  N., Zen, H., Graves, A., King, H., Walters, T., Belov, D., and Hassabis, D.
\newblock Parallel {W}ave{N}et: Fast high-fidelity speech synthesis.
\newblock In Dy, J. and Krause, A. (eds.), \emph{Proceedings of the 35th
  International Conference on Machine Learning}, volume~80 of \emph{Proceedings
  of Machine Learning Research}, pp.\  3918--3926, Stockholm Sweden, 10--15 Jul
  2018. PMLR.
\newblock URL \url{http://proceedings.mlr.press/v80/oord18a.html}.

\bibitem[Yin et~al.(2017)Yin, Kann, Yu, and
  Sch{\"{u}}tze]{DBLP:journals/corr/0001KYS17}
Yin, W., Kann, K., Yu, M., and Sch{\"{u}}tze, H.
\newblock Comparative study of {CNN} and {RNN} for natural language processing.
\newblock \emph{CoRR}, abs/1702.01923, 2017.
\newblock URL \url{http://arxiv.org/abs/1702.01923}.

\bibitem[Yu \& Koltun(2016)Yu and Koltun]{yu2016multiscale}
Yu, F. and Koltun, V.
\newblock Multi-scale context aggregation by dilated convolutions, 2016.

\bibitem[Zhou(2020{\natexlab{a}})]{ZHOU2020319}
Zhou, D.-X.
\newblock Theory of deep convolutional neural networks: Downsampling.
\newblock \emph{Neural Networks}, 124:\penalty0 319 -- 327, 2020{\natexlab{a}}.
\newblock ISSN 0893-6080.
\newblock \doi{https://doi.org/10.1016/j.neunet.2020.01.018}.
\newblock URL
  \url{http://www.sciencedirect.com/science/article/pii/S0893608020300204}.

\bibitem[Zhou(2020{\natexlab{b}})]{ZHOU2020787}
Zhou, D.-X.
\newblock Universality of deep convolutional neural networks.
\newblock \emph{Applied and Computational Harmonic Analysis}, 48\penalty0
  (2):\penalty0 787 -- 794, 2020{\natexlab{b}}.
\newblock ISSN 1063-5203.
\newblock \doi{https://doi.org/10.1016/j.acha.2019.06.004}.
\newblock URL
  \url{http://www.sciencedirect.com/science/article/pii/S1063520318302045}.

\end{thebibliography}
\bibliographystyle{packages/icml2021}

\onecolumn
\appendix
\chead{\textbf{Approximation Theory of CNN for Times Series Appendix }}
\begin{center}
    {\Large  Apppendix}
\end{center}
\section{Tensors and Higher Order Singular Value Decomposition}
As a higher order analogue of Singular Value Decomposition (SVD) for matrices,
Higher Order Singular Value Decomposition (HOSVD) for tensors
is the main tool to develop our theorems. 
In this section, we give a quick review of tensors,
and introduce the essential part of HOSVD which helps better understand the theorems.
The contents here are based on \cite{DeLathauwer2000} and \cite{Kolda2009}. 

\paragraph{Basic Notations}
Below is a table for different notations.

\begin{center}
\begin{tabular}{ |c|c|c| } 
 \hline
Type & Notation & Examples \\ 
\hline
 Tensor & Boldface Euler script letter & $\bm{\mathcal A}$ \\ 
 Matrix & Boldface capital letter & $\bm A$ \\
 Vectors &  Boldface lowercase letters & $\bm a$ \\
 Scalars & Lowercase letters & $a$\\
  \hline
\end{tabular}
\end{center}
The \textit{order} of a tensor is the number of its dimensions, which is also called modes. 
We use subscripts on corresponding notations to denote a specific slice of a tensor.
For example, 
$\tensor A \in \mathbb R^{4\times2\times4\times5}$ is an order 4 tensor,
and $a_{ijkl}$ denotes the $(i,j,k,l)$-element of $\tensor A$.
The vector coordinated along the first dimension is denoted as $\bm a_{:jkl}$, and the matrix coordinated along the first and the second dimension is denoted as
$\bm A_{::kl}$.

The norm of a tensor $\tensor A \in \mathbb R^{I_1 \times I_2 \times \cdots \times I_K}$ is defined by 
\begin{equation}
    \norm{\tensor A} = \sqrt{\sum_{i_1=1}^{I_1} \cdots \sum_{i_K=1}^{I_K}a^2_{i_1 \cdots i_K}},
\end{equation}

Let $\tensor A \in \mathbb R^{I_1\times\cdots I_K}$ and $\tensor B \in \mathbb R^{J_1\times \cdots \times 
J_{K'}}$ be two tensors, then the outer product of $\tensor A$,  $\tensor B$ is
a tensor in $\mathbb R^{I_1\times\cdots \times I_K\times J_1\times \cdots \times 
J_{K'}}$,
which is defined as
\begin{equation}
    (\tensor A \otimes \tensor B)_{i_1\dots i_Kj_1\dots j_{K'}}
    = a_{i_1\dots i_K}b_{j_1\dots j_{K'}}.
\end{equation}


\paragraph{Tensor reshaping}
The tensor reshaping refers that one can reshape a tensor into a matrix, or a matrix into a tensor based on some specific rules.
The process of reshaping a tensor into a matrix is called tensor flattening, while reshaping a matrix into a tensor is called tensorisation.
Tensor reshaping is the core idea of HOSVD.

\begin{definitionAppendix}
Let $\tensor A \in \mathbb R^{I_1 \times I_2 \times \cdots \times I_K}$
be an order K tensor.
The mode-k flattening of $\tensor A$ is denoted as $\bm A_{(k)} \in \mathbb R^{I_k \times (I_1 \cdots I_{k-1} I_{k+1}\cdots I_{K})} $, where the $(i_1,i_2,\cdots,i_K)$-element of $\tensor A$ is mapped to the $(i_k, j)$-element of the matrix $\bm A_{(k)}$ with
\begin{equation}
    j = 1+\sum_{\substack{s=1\\ s\neq k}}^K \Bigg ((i_s-1) \prod_{\substack{s'=1\\s'\neq k}}^{s-1}I_{s'} \Bigg).
\end{equation}
  That is, the columns of $\bm A_{(k)}$ are actually the vectors $\bm a_{i_1\cdots i_{k-1}: \ i_{k+1}\cdots i_K}$.
\end{definitionAppendix}
\begin{definitionAppendix}
Let $\bm A \in \mathbb R^{I_K \times ( I_1I_2\cdots I_{K-1})}$ be a matrix. 
Then
$\tensor A \in \mathbb R^{I_1 \times I_2 \times \cdots \times I_K} $ is the  tensorisation of $\bm A$ 
if its mode-K flattening equals to $\bm A$.
\end{definitionAppendix}
We illustrate the above two definitions by an example.

\begin{exampleAppendix}
    Consider an order 3 tensor $\tensor A \in \mathbb R^{4\times 3 \times 2}$ such that
    \begin{equation}
        \bm A_{::1} = 
        \begin{pmatrix}
            1 &2& 3\\
            4 &5&6\\
            7 & 8 & 9\\
            10&11&12
        \end{pmatrix}, \quad
        \bm A_{::2} = 
        \begin{pmatrix}
            13 &14& 15\\
            16 &17&18\\
            19 & 20 & 21\\
            22&23&24
        \end{pmatrix}.
    \end{equation}
    Then
    \begin{equation}
    \begin{gathered}
        \bm A_{(1)} =
        \begin{pmatrix}
        1 & 2 &3& 13 & 14 & 15  \\
        4 & 5 & 6 & 16 & 17 & 18  \\
        7&8&9&19&20&21\\
        10&11&12&22&23&24
        \end{pmatrix},\\
        \bm A_{(2)} =
        \begin{pmatrix}
        1 & 4 &7& 10 & 13 & 16&19&22  \\
        2 & 5 & 8 & 11 & 14 & 17 &20&23  \\
        3&6&9&12&15&18&21&24
        \end{pmatrix},\\
    \bm A_{(3)} =\begin{pmatrix}
    1 & 4 & 7 &10 &\cdots & 3 & 6 & 9 & 12\\ 
    13 & 16 & 19 &22& \cdots  & 15 & 18 & 21 & 24
    \end{pmatrix}.
    \end{gathered} 
    \end{equation}
    
Conversely, $\tensor A$ is the tensorisation of $\bm A_{(K)}$ in 
$\mathbb R^{4\times 3 \times 2}$.
\end{exampleAppendix}

Recall the tensorisation $T_{l^K}(\bm a)$ we defined in the paper, where $\bm a \in \mathbb R^{l^K}$ is a vector. 
In this case, we first rearrange the vector $\bm a \in \mathbb R^{l^K}$ into a matrix $\bm A \in \mathbb R^{l\times l^{K-1}}$ according to row major ordering, 
then $T_{l^K}(\bm a)$ is the tensor in $\mathbb R^{l\times l\cdots \times l}$ defined as the tensorisation of $\bm A$.

\paragraph{Singular values and the rank of tensors}
\begin{definitionAppendix}\label{def:sv_tensor}
The singular values of a tensor  $\tensor A\in \mathbb R^{I_1 \times \cdots \times I_K}$ are defined as
\begin{equation}
    \set{\sigma_{i_{k}}\depen{k}}, \quad k = 1,\dots,K,
\end{equation}
where $\set{\sigma_{i_k}\depen{k}}$ are the singular values of $\bm A_{(k)}$ arising from the SVD for matrices.
That is, the singular values of $\tensor A$ is a collection of all the singular values of its mode-$k$ flattening.
\end{definitionAppendix}
\begin{remarkAppendix}
    The norm of $\tensor A$ satisfies 
    \begin{equation}\label{tensor norm}
       \norm{\tensor A} = \norm{\bm A_{(k)}}_F = \sqrt{{\sum_{i_k=1}^{r_k}{\left(\sigma_{i_k}\depen{k}\right)}^2}},
     \quad \quad k = 1,2, \dots, K,
    \end{equation}
    where $r_k$ is the matrix rank of $\bm A_{(k)}$.
\end{remarkAppendix}

There are various ways to define the rank of a tensor.  
Here we use the following definition.
\begin{definitionAppendix}
The rank of a tensor $\tensor A$ is defined as 
\begin{equation}
    \rank{\tensor A} := \sum_{k=1}^K r_k.
\end{equation}
Here $r_k$ is the matrix rank of $\bm A_{(k)}$, which is also called the $k$-rank of $\tensor A$.
Recall that the matrix rank equals to the number of its non-zero singular values, it follows from Definition \ref{def:sv_tensor} that the tensor rank also equals to the number of its non-zero singular values.
\end{definitionAppendix}

\begin{remarkAppendix} \label{rmk: hosvd form}
    Recall that the matrix SVD can be written in the form of outer product:
$
   \displaystyle \bm A = \sum_{i=1}^{r} \sigma_i \bm u_i \bm v_i^\top
$ with $r = \rank \bm A$.  
This can be generalised to HOSVD. For any $\tensor A \in \mathbb R^{I_1 \times I_2 \times \cdots \times I_K}$, we have
\begin{equation}
    \tensor A = \sum_{i_1=1}^{r_1}\sum_{i_2=1}^{r_2}\cdots \sum_{i_K=1}^{r_K}
    s_{i_1i_2\dots i_K} \bm u^{(1)}_{i_1} \otimes \bm u^{(2)}_{i_2}\otimes
    \cdots \otimes \bm u^{(K)}_{i_K},
\end{equation}
where $r_k$ is the $k$-rank of  $\tensor A$,  $s_{i_1i_2\dots i_K} \in \mathbb R$ and ${\bm u_{i_k}^{(k)}} \in \mathbb R^{I_k}$.
\end{remarkAppendix}

We end up with the following approximation property of HOSVD, which can be viewed as an analogue of Eckart-Young-Mirsky theorem for matrix SVD.
\begin{propositionAppendix} \label{prop: hosvd error}
Let $\tensor A, \tensor{\hat A}\in \mathbb R^{I_1 \times \cdots \times I_K}$ with the $k$-ranks denoted as $r_1,\dots,r_K$ and $r'_1,\dots, r'_K$
respectively.
Let $\sigma_1\depen{k} \geq \sigma_2\depen{k}\geq \cdots \geq \sigma_{r_k}\depen{k}\geq 0$ be the singular values of $\bm A_{(k)}$,
we have
\begin{equation}
    \inf_{\tensor{\hat A}}~\norm{\tensor A - \tensor{\hat A}}^2 \leq \sum_{i_1 = r'_1 + 1}^{r_1} \left(\sigma^{(1)}_{i_1}\right)^2 + 
    \sum_{i_2 = r'_2 + 1}^{r_2} \left(\sigma^{(2)}_{i_2}\right)^2+ \cdots +
    \sum_{i_K = r'_K + 1}^{r_K} \left(\sigma^{(K)}_{i_K}\right)^2,
\end{equation}
where the infimum is taken over $\tensor{\hat A}\in\mathbb R^{I_1 \times \cdots \times I_K}$ such that 
$r'_k\le r_k$, $k=1,\dots K$.

\end{propositionAppendix}

\begin{remarkAppendix}
  Note that for matrix SVD, Eckart-Young-Mirsky theorem shows that the approximation error equals to the tail sum of singular values, but for HOSVD only the upper bound holds.  
\end{remarkAppendix}

\section{Proofs}
Now we get down to prove all the results shown in the text.

Recall that the input space $\mathcal X$ is a Hilbert space, 
one can apply the standard representation theorem.
\begin{theoremAppendix}
{\normalfont{(Riesz Representation Theorem)}}
For any continuous linear functional $H$ defined on $\mathcal X$,
there exists a unique $\rho \in \mathcal X$
such that
	 \begin{equation}
	 	H(\bm x) = \sum_{s = -\infty}^{\infty} \rho(s)^\top x(s),
	 \end{equation}
and
\begin{equation}
    \norm{H} = \norm{\rho}_{\mathcal X}.
\end{equation}
\begin{proof}
See \citet{Bramwell1979}, Theorem 3.8-1.
\end{proof}
\end{theoremAppendix}

Based on this, we prove Lemma \ref{lmm:representation}.

\begin{proof}[Proof of Lemma \ref{lmm:representation}]
    By Riesz Representation Theorem, for any $t \in \mathbb Z$ and $H_t \in \bm H$, there exists a unique $\bm \rho_t \in \mathcal X$ such that
    \begin{equation}
        H_t(\bm x) = \sum_{s=-\infty}^{\infty}\rho_t(s)^\top x(s).
    \end{equation}
    With the fact that $\bm H$ is causal, we have
    \begin{equation}
        H_t(\bm x) = \sum_{s=-\infty}^{t}\rho_t(s)^\top x(s).
    \end{equation}
    By the time homogeneity $H_t(\bm x) = H_{t+\tau}(\bm x\depen{\tau})$ with $\tau=-t$, we get
    \begin{equation}
         H_t(\bm x) = \sum_{s=-\infty}^{t}\rho_t(s)^\top x(s) =
        \sum_{s=-\infty}^{0}\rho_{0}(s)^\top x(s+t).
    \end{equation}
    The conclusion follows by taking $\rho\depen{\bm H}(s) = \rho_0(-s)$.
\end{proof}

Recall the example in section 4.2, where we showed that
for any matrix $\bm A$ with the rank no more than $2$, there exists
$\rhh \in \hcnn^{(2,2,\{M_k\})}$ such that $T_{2^2}(\rhh)=\bm A$.
Now we extend this to the general $l$ and $K$.

\begin{propositionAppendix}\label{prop: outer product form}
For any $\bm w_1,\bm w_2,\dots,\bm w_K \in \mathbb R^l$, we have
    \begin{equation}\label{eq: conv is outerproduct }
        T_{l^K}\Big(\bm w_K \Conv_{l^{K-1}}\bm w_{K-1}\Conv_{l^{K-2}}\cdots
    \Conv_{l} \bm w_1\Big ) = 
    \bm w_K \otimes \bm w_{K-1} \otimes \cdots \otimes \bm w_1.
    \end{equation}
\begin{proof} 
    We prove this by induction.
    When $K=2$, it is the usual outer product for vectors.
    Suppose the conclusion holds for $K$, we prove for $K+1$. 
    Let $\bm f_K := \bm w_K \Conv_{l^{K-1}}\bm w_{K-1}\Conv_{l^{K-2}}\cdots
    \Conv_{l} \bm w_1  = (c_1,c_2,\dots,c_{l^K})$ 
    and 
    $\bm w_{K+1} = (w_{1},w_{2},\dots, w_{l})$,
    then 
    \begin{align}
        (\bm w_{K+1} \Conv_{l^{K}} \bm f_K)(t) &= \sum_{s\in\mathbb N} w_{K+1}(s)f_K(t-l^K s) \nonumber \\
        & = w(0)f_K(t) + w(1)f_K(t-l^K)  +\cdots +  w(l-1)f_K(t-(l-1) l^K)
        \nonumber \\
        & = w_1f_K(t) + w_2f_K(t-l^K)  +\cdots +  w_lf_K(t-(l-1) l^K).
    \end{align}
    The right hand side is non-zero only when $t = 0, 1, 2, \dots, l^{K+1}-1$, which gives
    \begin{equation}\label{eq: dilated conv}
        (\bm w_{K+1} \Conv_{l^{K}} \bm f_K) = 
        (w_1c_1,w_1c_2,\dots, w_1c_{l^K}, w_2c_1, w_2c_2, \dots, w_2c_{l^K}, \dots,w_lc_1,w_lc_2 \dots,w_lc_{l^K})
        \in \mathbb R^{l^{K+1}}.
    \end{equation}
    By the tensorisation for vectors discussed above, $\bm w_{K+1} \Conv_{l^{K}} \bm f_K$ can be rearranged into a matrix according to  row major ordering:
    \begin{equation}
      \begin{pmatrix}
          w_1c_1 & w_1c_2 & \cdots & w_1c_{l^K}  \\
          w_2c_1 & w_2c_2 & \cdots & w_2 c_{l^K} \\
          \vdots & \vdots & \ddots & \vdots \\
          w_lc_1 & w_lc_2 & \cdots & w_lc_{l^K} 
      \end{pmatrix}\in\mathbb R^{l\times l^K}.
    \end{equation}
  This is in fact the mode-K flattening of $ \bm w_{K+1} \otimes T_{l^K}(\bm f_K) $.
  By the induction hypothesis, we conclude that
  \begin{equation}
      T_{l^{K+1}}\Big(\bm w_{K+1} \Conv_{l^{K}} \bm f_K\Big) =  
      \bm w_{K+1} \otimes T_{l^K}(\bm f_K) =  \bm w_{K+1} \otimes \bm w_K \otimes \bm w_{K-1} \otimes \cdots \otimes \bm w_1.
  \end{equation}
    \end{proof}
\end{propositionAppendix}

\begin{propositionAppendix}\label{prop: Cnn tensor form}
Let $\tensor A \in \mathbb R^{l \times l \times \cdots \times l}$ be an order $K$ tensor with the $k$-rank $r_k$, $k=1,\cdots,K$.
There exists $\rhh \in \hcnn\depen{l}$
such that $T_{l^K}(\rhh)=\tensor A$.
\begin{proof}
  Recall the linear CNN model
  \begin{equation}
		\begin{aligned}
		\bm h_{0,i} &= \bm x_i,\\
		\bm h_{k+1,i} &= \sum_{j=1}^{M_k} {\bm w}_{kji} \Conv_{l^k} \bm h_{k,j} ,\\
		\bm {\hat y} &= \bm h_{K+1}.
		\end{aligned}
\end{equation}
By linearity, there exists $\rhh \in \hcnn^{(l,K,\set{M_k})}$ such that
\begin{equation}
    \rhh=
    \sum_{i_{K}=1}^{M_K}\sum_{i_{K-1}=1}^{M_{K-1}}\cdots\sum_{i_1=1}^{M_1} 
    \bm w_{K, i_{K}} \Conv_{l^{K}} \bm w_{K-1, i_{K-1}}\Conv_{l^{K-1}} \cdots
    \Conv_{l}\bm w_{1, i_1},
\end{equation}
where $\bm w_{k,i_k} \in \set{\bm w_{kij}}$ is a filter at layer $k$.
According to Proposition \ref{prop: outer product form}, we have
\begin{equation}
    T_{l^K}(\rhh)=
   \sum_{i_{K}=1}^{M_K}\sum_{i_{K-1}=1}^{M_{K-1}}\cdots\sum_{i_1=1}^{M_1}
   \bm w_{K,i_{K}} \otimes \bm w_{K-1, i_{K-1}}  \otimes \cdots \otimes \bm w_{1,i_1}.
\end{equation}

The conclusion now follows from Remark \ref{rmk: hosvd form} for sufficient large $M_k$.
\end{proof}
\end{propositionAppendix}

Now we can prove the first main theorem in the main text.

	\begin{proof}[Proof of Theorem \ref{thm:UAPCNN}]
	By Lemma \ref{lmm:representation} and (12) in the main text, we have
    \begin{align}\label{eq:errUb}
        \norm{\bm H - \bhh}^2 = \sup_{t \in \mathbb Z} \norm{H_t - \widehat H_t}^2
        &=
         \sup_{t \in \mathbb Z}\sup_{\norm{\bm x}_{\mathcal X}\leq 1} 
         \left\lvert{\sum_{s\in\mathbb N}\big(\rh(s)-\rhh(s)\big)^\top x(t-s)}\right\lvert^2 \nonumber \\
         &\leq \sum_{s\in\mathbb N}\big \lvert \rh(s)-\rhh(s) \big\lvert^2 \nonumber \\
          & = \sum_{s=0}^{l^K-1}\Big\rvert \rh(s) - \rhh(s)\Big\rvert^2
    + \sum_{s=l^K}^{\infty} \Big|  \rh(s) \Big|^2 \nonumber \\
    & = \left\|T_{l^K}(\rhh) - T_{l^K}(\rh)\right\|^2 
    + \sum_{s=l^K}^{\infty} \Big|  \rh(s) \Big|^2.
    \end{align}
    Since $\rh \in \ell^2$,
	we can choose $K$ appropriately large such that the second term is less than $\epsilon$. According to Proposition \ref{prop: Cnn tensor form}, there exists $\hcnn\depen{l}$ such that the first term is zero. The proof is completed.
	\end{proof}
Lemma \ref{lmm:lowrankapp_tensor} follows from Proposition \ref{prop: hosvd error} and the definition of rank for tensors.

This was denoted by $\|\cdot \|_{l,g}$ in the main text. Here we amend the notation as we do not discuss the norm aspects.

\begin{propositionAppendix}\label{prop: addition zeros}
Let $\bm \rho$ be a finitely supported sequence
such that $r(\bm \rho) \leq l^K-1 $.
Denote the singular values of $\ten{\bm \rho}$ by $\sigma_1\ge\sigma_2\ge\cdots\ge\sigma_{lK}$.
For any sequence $\bm{\hat \rho} \in \mathbb R^{l^{K+1}}$ with $\bm{\hat \rho}_{[0,l^K-1]} = \bm\rho$ and $\bm{\hat \rho}_{[l^K,l^{K+1}]} = 0$, the singular values of $T_{l^{K+1}}(\hat{\bm \rho})$ are
\begin{equation}
   \left\|\ten{\bm\rho}\right\| = \sigma_{l(K+1)}  \geq\sigma_1\ge\sigma_2\ge\cdots\ge\sigma_{lK}\geq 0 = 
   \sigma_{lK+1} =\sigma_{lK+1} = \cdots = \sigma_{lK+l-1}.
\end{equation}
\begin{proof}
The singular values of $T_{l^{K+1}}(\hat{\bm \rho})$ arising from mode-$1$ to mode-$K$ flattening are not changed,
since there are only additions of zero columns. 
Now we consider the mode-$(K+1)$ flattening.
We add zeros to an additional dimension of the tensor
such that $\hat a_{i_1i_2\dots i_K 1} = a_{i_1i_2\dots i_K}$
and $\hat a_{i_1i_2\dots i_K j} = 0 $ for $2\leq j\leq l$.
Then the columns of $T_{l^{K+1}}(\hat{\bm \rho})_{(K+1)}$ are the vectors along the new dimension
$(\hat a_{i_1i_2\dots i_K 1}, 0,0,\dots,0) = (a_{i_1i_2\dots i_K}, 0,0,\dots,0)$, which gives a unique non-zero singular value 
$ \sigma_{l(K+1)}=\norm{\ten{\bm\rho}}$, and $(l-1)$ zero singular values.
\end{proof}
\end{propositionAppendix}

\begin{propositionAppendix}\label{prop: finite support in g}
Suppose the function $g$ is monotonously decreasing and strictly positive.
Then for any $\bm H\in\cs$ such that $\rh$ is finitely supported, 
we have $\bm H \in \cs\depen{l,g}$.
\begin{proof}
Let $K'=\inf\{K\in\mathbb N_+: l^{K}\ge r(\rh)\}$. Then for any $s \ge  lK'$,
$\sum_{i=s+K'+k}^{l(K'+k)}\abs{\sigma_i\depen{K'+k}}^2 = 0$ for all $k\in\mathbb N_+$ according to Proposition \ref{prop: addition zeros}, which completes the proof.
\end{proof}
\end{propositionAppendix} 
\begin{propositionAppendix}
Suppose $\bm H\in\cs$ with $\rh$ finitely supported.
Then there exists a finitely supported decreasing $g$ 
such that $\bm H \in \cs\depen{l,g}$.
\begin{proof}
The proof is a straightforward application of Proposition \ref{prop: addition zeros}.
\end{proof}
\end{propositionAppendix}

Next we prove the main theorem on error bound.
\begin{proof}[Proof of theorem \ref{thm: error bound}] (i) Lower bound. Since
    \begin{align}
        \norm{\bm H - \bhh}
        &=
         \sup_{t \in \mathbb Z}\sup_{\norm{\bm x}_{\mathcal X}\leq 1} 
         \left\lvert{\sum_{s\in\mathbb N}\big(\rh(s)-\rhh(s)\big)^\top x(t-s)}\right\lvert \nonumber,\\
         \intertext{by taking a specific $\bm x$ with $\displaystyle x_i(0) = \frac 1 {\sqrt d} \sgn (\rh_i(t)-\rhh_i(t))$ and $x_i(s) = 0$ otherwise, we have}
         & \ge \frac 1 {\sqrt d}\sup_{t \in \mathbb N}
         {\left\lvert\rh(t)-\rhh(t)\right\lvert}  \nonumber\\
         &\geq \frac 1 {\sqrt d} \sup_{t \in [l^K,\infty]} 
         {
         \left\lvert\rh(t)\right\lvert}
         \nonumber \\
         & \ge \frac 1 {\sqrt d} \sup_{t \in [l^K,\infty]} \norm{\rh(t)}_2,
    \end{align}
    where the inequality holds by taking a specific $\bm x$ with $x_i(0) =1$  and $x_i(s) = 0$ otherwise.
    

(ii) Upper bound. Following the proof of Theorem \ref{thm:UAPCNN}, or \eqref{eq:errUb} gives
\begin{align}
        \norm{\bm H - \bhh} 
    & \leq \sum_{i=1}^d \left\|T_{l^K}(\rhh_i) - T_{l^K}(\rh_i)\right\| + \norm{\rh_{[l^K,\infty]}}_2.
    \end{align}             
The remaining task is to bound the first term.
Based on Lemma \ref{lmm:lowrankapp_tensor}, 
we only need to calculate the maximum possible rank of $T_{l^K}(\rhh_i)$.
Let $r_k$  denote  the $k$-rank of $T_{l^K}(\rhh_i)$.
From Remark \ref{rmk: hosvd form}, by absorbing the scalar $s_{i_1i_2\dots i_K}$ into any of the vector $\bm u_{i_k}^{(k)}$,
we have the following relationship:
\begin{align}
    d\prod_{k=1}^K r_k + lK\geq d\prod_{k=1}^K r_k + \sum_{k=1}^K r_k \geq \sum_{k=2}^K M_k M_{k-1},
\end{align}
thus we have
$\prod_{k=1}^K r_k  \geq  \frac 1 d( \sum_{k=2}^K M_k M_{k-1} -lK)=M$,
which implies $\rank T_{l^K}(\rhh_i) = \sum_{k=1}^{K}r_k \geq KM^{\frac 1 K}$.
Combined with Definition \ref{def:Hnorm} gives the conclusion.
\end{proof}

Next we look in details the two examples about comparison between RNNs and CNNs in details.

\paragraph{Example where RNNs out-perform CNNs.}
We take a scalar input with $d=1$.
Consider a target $\bm H\in \cs$ 
with the representation $\rh(t) = \gamma^{t}$, where $ 0<\gamma<1.$
It is easy for RNNs to approximate this target, since the representation has a power form. In fact, we have $\bm H \in \hrnn^{(1)}$, i.e. a RNN with one hidden unit is sufficient to achieve an exact representation
for any $\gamma\in(0,1)$.

For any CNN model $\bhh \in  \hcnn\depen{l,K,\set{M_k}}$, based on the lower bound of Theorem \ref{thm: error bound},  we have that
\begin{equation}
    	\norm{\bm H - \bhh}^2 \geq  \frac 1 {\sqrt d} \sup_{t \in [l^K,\infty]} \norm{\rh(t)}_2.
\end{equation}

Thus, in order to achieve an approximation error with $\norm{\bm H - \bhh} < \epsilon$, we have
\begin{equation}
     \sup_{t \in [l^K,\infty]} \norm{\rh(t)}_2  = \gamma^{l^K}
	< \epsilon.
\end{equation}
This implies $l^K \geq \frac{\log ( \epsilon )}{\log (\gamma )}$.
That is, the number of layers necessary to 
achieve an approximation error smaller than $\epsilon$ 
diverges to infinity as $\gamma$ approaches 1.

\paragraph{Example where CNNs out-perform RNNs.}
We still take a scalar input with $d=1$.
Consider a target $\bm H \in \cs$ with the  representation
\begin{equation}
    \rho^{(\bm H)}(t) = 
    \begin{cases}
        1, & t = 2^K \\
        0, & \text{otherwise}
    \end{cases}, \quad  K \in \mathbb N_+.
\end{equation}
We have $\bm H \in \hcnn^{(2,K,\set{1})}$.
That is, a $K$-layer CNN with one channel per layer
is sufficient to achieve an exact representation. 

Recall that RNN approximates the target $\rh$ with a power sum $\rhh(s) = c^\top W^{s-1}U$.
Suppose here $W \in \mathbb R^{m \times m}$ is a diagonalisable matrix with negative eigenvalues.
It has some special structures which are summarised in the following theorem.
\begin{theoremAppendix} \label{thm: power sum} \cite{Borwein1996}
	Let $\displaystyle E_m :=\left \{u: u(t)=c_0 + \sum_{i=0}^{m} c_i\gamma_i^{ t} , \  c_i\in \mathbb R, \gamma_i > 0 \right\}$, then
	\begin{equation}
		\sup_{u\in E_k} \frac{\abs{u'(y)}}{\sup_{s\in [a,b]} u(s)}\leq \frac {2m-1}{\min\{y-a,b-y\}}, \ \ y \in (a,b).
	\end{equation}
\end{theoremAppendix}
We rewrite this theorem into a discrete form.
\begin{corollaryAppendix}\label{coro: power sum}
Let $u \in E_m$. Then
	\begin{equation}
		\abs{u(t+1)-u(t)}\leq \frac{2m}{t}\sup_{s\geq 0} u(s).
	\end{equation}
	\begin{proof}
	By the mean value theorem, there exists $y \in [t,t+1]$ such that $\abs{u(t+1)-u(t)}  = \abs{u'(y)}$. 
	The corollary then follows from Theorem \ref{thm: power sum}.
	\end{proof}
\end{corollaryAppendix}

For a fixed $m$,
the changes between $u(t+1)$ and $u(t) $ approaches zero
as $t$ goes to infinity.
This implies that if there is a sudden change in $u$ far from the origin, the number of terms $m$ must be large.

In order to achieve an approximation error with $\norm{\bm H - \bhh} < \epsilon$,
by taking a specific $\bm x$ as the unit sample function where $x(0) = 0$ and $x(s) = 0$ otherwise  , we have
\begin{align}
			\epsilon> \sup_t \sup_{\norm {\bm x }\leq 1} \abs{H_t(\bm x)- \hat H_t(\bm x)} &\geq  \sup_t \abs{H_t(\bm x)- \hat H_t(\bm x)} \\
			&=\sup_t \abs{\rh(t)- c^\top  W^{t-1}U }\\
			&=\sup_t \abs{\rh(t)- u(t) },  \quad u\in E_{m^2}.\\
			\intertext{Since}
			\abs{u(2^K+1)} &< \epsilon,\\
			\abs{u(2^K)-1} &< \epsilon,\\
		\intertext{we have}
			\abs{u(2^K+1)-u(2^K)}&=\abs{u(2^K+1)-1-u(2^K)+1}\\
			&>1-\abs{u(2^K+1)}-\abs{u(2^K)-1}\\
			&>1-2\epsilon
		\end{align}
		Combining with Corollary \ref{coro: power sum} gives
		\begin{align}
			m^2 > 2^{K-1}\frac{ 1-2\epsilon}{1+\epsilon}.
		\end{align}
As $K$ increases,
the number of parameters needed for RNNs
to achieve an error less than $\epsilon$
increases exponentially, 
while this increment is linear for CNNs.

\section{Special structures of dilated convolutions }

In this section, we discuss an interesting structure of dilated convolutions.
\begin{propositionAppendix}\label{thm: radix}
    Let $\bm w_1,\dots,\bm w_K$ be $K$ filters with the same filter size $l$, $\bm s = (s_1,s_2,\dots, s_K)$ with $0\leq s_k \leq l-1$, $k=1,\cdots,K$.
    Suppose all entries of $\bm w_k$ are zero except $w_k(s_k) = 1$, then
    \begin{equation}
        (\bm w_K \Conv_{l^{K-1}} \bm w_{K-1} \Conv_{l^{K-2}}\cdots 
        \Conv_{l} \bm w_{1})(t) =
        \begin{cases} 
            1, & t =  \hat t \\
            0, & \text{otherwise}
        \end{cases},
    \end{equation}
    where $\displaystyle \hat t = (s_Ks_{K-1}\dots s_{1})_l:=\sum_{i=0}^{K-1} s_{i+1} l^{i}$. That is,  
    $\hat t$ can be written as a base $l$ expansion with digits $s_K$ to $s_1$.
    \begin{proof}
    We prove this by induction.
    When $K=1$, the conclusion is obvious.
    Suppose the conclusion holds for $K$, then by \eqref{eq: dilated conv},
    \begin{equation}
        (\bm w_{K+1} \Conv_{l^{K}} \bm f_K)(t) = 
        (w_1c_1,w_1c_2,\dots, w_1c_{l^K}, w_2c_1, w_2c_2, \dots, w_2c_{l^K}, \dots,\dots,w_lc_1,w_lc_2, \dots,w_lc_{l^K}).
    \end{equation}
    Suppose $w_{K+1}(s_{K+1}) = c_m = 1$, 
    then the position of 1 in the above vector is $s_{K+1}\ l^K + c_m$, 
    which means the results also holds for $K+1$.
    \end{proof}
\end{propositionAppendix}
This result allows us to construct a filter with value $1$ at any specific position 
$\hat t$, by choosing filters $\{\bm w_k\}$ according to the 
base $l$ expansion of $\hat t$. 
We illustrate this by an example.
\begin{exampleAppendix}
    Let $l=4,K=3$ and $\bm w_1 = (0,0,0,1), \bm w_2 = (1,0,0,0), \bm w_3 = (0,1,0,0)$.
    The positions of value $1$ are recorded in $\bm s = (3,0,1)$. Then
    \begin{equation}
        (\bm w_{3} \Conv_{4^2} \bm w_{2} \Conv_{4^1}\bm w_{1})(t) = \begin{cases} 
            1, & t =  19 = (103)_4\\
            0, & \text{otherwise}
        \end{cases}
    \end{equation}
\end{exampleAppendix}

Based on above, one can define a notion of sparsity, 
which gives another sufficient condition for the exact representation.
\begin{definitionAppendix}
Suppose $\rh$ is a finitely supported sequence. The sparsity of $\rh$ is defined as the number of its non-zero elements,
which is denoted by $\norm{\rh}_0$.
\end{definitionAppendix}

\begin{corollaryAppendix}\label{coro: sparsity condition}
Suppose $\rh$ is a  finitely supported sequence with $r(\rh) \leq l^K-1$.
Then $K\norm{\rh}_0$ filters are sufficient to achieve an exact representation.
\begin{proof}
This follows from Theorem \ref{thm: radix}, where for each non-zero element, we can use $K$ filters to generate it.
\end{proof}
\end{corollaryAppendix}

In the main text, we use the condition that $K(M+1)^{\frac 1 K}\geq \rank \ten{\rh}$
to ensure an exact representation.
Instead of calculating the rank of $\ten{\rh}$,
Corollary \ref{coro: sparsity condition} gives us another way to decide the number of filters sufficient to have an exact representation.
This gives us the insight that if a target is sparse in the sense that $\norm{\rh}_0$ is small,
it can also be efficiently approximated by  CNNs.

\begin{remarkAppendix}
        Notice that the rank of a tensor not only depends on its sparsity, but also depends on specific positions of the non-zero elements. To illustrate, consider the following example
\begin{align}
	\bm \rho_1 = \begin{pmatrix}
		1 & 0 &
		1 & 0
	\end{pmatrix} & \text{ and }
	\bm \rho_2 = \begin{pmatrix}
		1 & 0 &
		0 & 1
	\end{pmatrix}.
\end{align}
Both of them have a sparsity $2$, but $\rank T_{2^2}(\rho_1) = 2$ while $\rank T_{2^2}(\rho_2) = 4$.
\end{remarkAppendix}

\end{document}